\documentclass[twoside]{article}

\usepackage[accepted]{aistats2020}
\usepackage[utf8]{inputenc} % allow utf-8 input
\usepackage[T1]{fontenc}    % use 8-bit T1 fonts
\usepackage{hyperref}       % hyperlinks
\usepackage{url}            % simple URL typesetting
\usepackage{color}
\usepackage{xcolor}
\usepackage{booktabs}       % professional-quality tables
\usepackage{amsfonts}       % blackboard math symbols
\usepackage{microtype}      % microtypography

\usepackage{times}
\usepackage{courier}
\usepackage{graphicx}
\usepackage{algorithm}
\usepackage{algcompatible}
\usepackage{lipsum}
\algnewcommand{\lst}{\texttt{lst}}
\algnewcommand{\slst}{\texttt{slst}}
\algnewcommand{\SEND}{\paragraph{send}}

\newsavebox{\algleft}
\newsavebox{\algright}
\usepackage{amsmath}
\usepackage{amssymb}
\usepackage{amsthm}
\usepackage{url}
\usepackage{bm}
\usepackage{color}
\usepackage{booktabs}
\usepackage{multirow}
\usepackage{multicol}

\usepackage{array}
\begin{document}
\newtheorem{prop}{Proposition}
\renewcommand{\algorithmicrequire}{\paragraph{Input:}}
\renewcommand{\algorithmicensure}{\paragraph{Output:}}
\makeatletter

\newtheorem{property}{Property}
\newtheorem{theorem}{Theorem}
\newtheorem{lemma}{Lemma}
\newtheorem{proposition}{Proposition}
\newtheorem{corollary}{Corollary}
\newtheorem{innercustomthm}{Theorem}
\newenvironment{customthm}[1]
  {\renewcommand\theinnercustomthm{#1}\innercustomthm}
  {\endinnercustomthm}



\twocolumn[

\aistatstitle{Online Binary Space Partitioning Forests}

\aistatsauthor{ Xuhui Fan \And Bin Li \And  Scott A. Sisson }

\aistatsaddress{School of Mathematics and Statistics\\University of New South Wales\\xuhui.fan@unsw.edu.au \And  School of Computer Science\\Fudan University\\libin@fudan.edu.cn \And School of Mathematics and Statistics\\University of New South Wales\\scott.sisson@unsw.edu.au} ]

\begin{abstract}
The Binary Space Partitioning-Tree~(BSP-Tree) process was recently proposed as an efficient strategy for space partitioning tasks. Because it uses more than one dimension to partition the space, the BSP-Tree Process is more efficient and flexible than conventional axis-aligned cutting strategies. However, due to its batch learning setting, it is not well suited to large-scale classification and regression problems. In this paper, we develop an online BSP-Forest framework to address this limitation. With the arrival of new data, the resulting online algorithm can simultaneously expand the space coverage and refine the partition structure, with guaranteed universal consistency for both classification and regression problems. The effectiveness and competitive performance of the online BSP-Forest is verified via simulations on real-world datasets.
\end{abstract}

%%%%%%%%%%%%%%%%%%%%%%%%%%%%%%%%%%
\section{Introduction}
%%%%%%%%%%%%%%%%%%%%%%%%%%%%%%%%%%

The BSP-Tree Process~\cite{xuhui2016OstomachionProcess,pmlr-v84-fan18b,pmlr-v89-fan18a} is a stochastic space partitioning process defined in a multi-dimensional space with a binary-partition strategy. Its general goal is to identify meaningful ``blocks'' in the space, so that data within each block exhibits some form of homogeneity. Similar to other space partitioning processes~\cite{kemp2006learning,roy2009mondrian,nakano2014rectangular,NIPS2018_RBP}, the BSP-Tree Process can be applied in many areas, including relational modeling~\cite{kemp2006learning,airoldi2009mixed,fan2016copula,SDREM}, community detection~\cite{nowicki2001estimation,karrer2011stochastic}, collaborative filtering~\cite{porteous2008multi,Li_transfer_2009}, and random forests~\cite{LakRoyTeh2014a}. 

Instead of the axis-aligned cuts adopted in most conventional approaches~\cite{kemp2006learning,roy2009mondrian,nakano2014rectangular}, the BSP-Tree Process implements oblique cuts (in more than one dimension) to recursively partition the space into new sub-spaces. 
In this way, it can describe the dimensional dependence more efficiently, in terms of fewer cuts or improved prediction performance. In addition, the BSP-Tree Process has the attractive theoretical property of self-consistency. Based on this property,  a projective system~\cite{CraneProjectiveSystem} can be constructed to ensure distributional invariance, when restricting the process from a larger domain to a smaller one, and safely extend a finite domain to multi-dimensional infinite space. 

Despite the existing clear potentials, there are two main challenges in the BSP-Tree Process. (1) {\it Scalability}: Its batch learning mode is unappealing given the practical real, large-scale datasets, as multiple scans over the data are required. (2) {\it Theoretical properties}: 
Moving from a batch to an online learning algorithm, rigorous universal consistency properties are required to ensure the theoretical correctness of the underlying BSP-Tree Process prior as the domain varies with the arrival of new data. This can be challenging due to
the recursive structure of the partition strategy and the Markov property of partition refinement.
% make it difficult to examine important factors, e.g.~the diameter of the blocks and the number of blocks. 

%In terms of the challenge in scalability, w
In this paper we propose an online BSP-Forest framework to address these challenges.
Instead of batch model training, our framework sequentially incorporates each data point and updates the model without the need for data labels. The tree structure update in the forest is implemented in two stages. {\em (i) Node coverage expansion:} Within the nodes~(all of nodes are convex polyhedron-shaped) of each tree, we focus on the convex hull that contains all of its allocated datapoints. When a new point arrives in this node but falls outside the convex hull, we enlarge the hull and change the tree structure. {\em (ii) Refinement over existing partitions:} We use a monotonically increasing budget sequence for new node creation, meaning that more data tend to result in finer partitions. These stages are each supported by the self-consistency and Markov properties of the BSP-Tree Process. In addition, we demonstrate its universal consistency by proving that the online BSP-Forest will converge to the Bayesian error for all possible label distributions. The effectiveness and competitive performance of the online BSP-Forest in classification and regression is verified through extensive simulation studies.

%\textcolor{red}{[This increasing budget sequence seems to be a weakness, as (example) you can then get a different partition if exactly the same data arrive in a different order, or (example) for data with one trye node, you will over partition as you can't simplify two nodes into a single node. Maybe this point needs including and elaborating on in the Discussion section.]}

%The consistency result of the online BSP Forest is guaranteed by further investigating its theoretical behaviour. Through displaying the expected value of the diameter in a block and the number of blocks, we have display a universal consistency result. Further, a minimax rate convergence for the Lipschitz function can  also obtained. 
%
%All  proofs are provided in the Supplementary Material.

\begin{figure*}[t]
\centering
\includegraphics[width =  0.4 \textwidth]{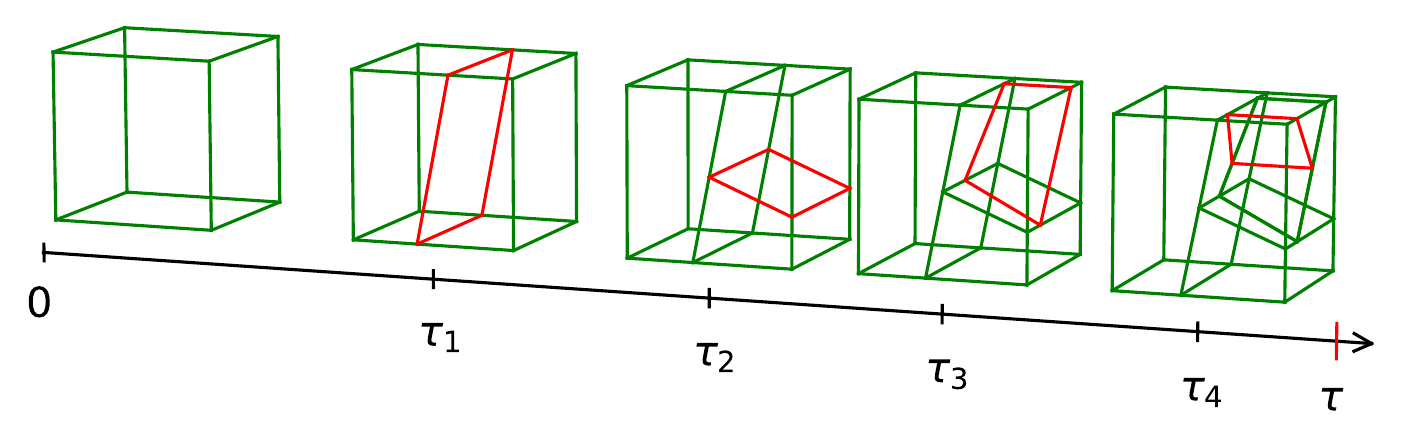}\quad
\includegraphics[width =  0.55 \textwidth]{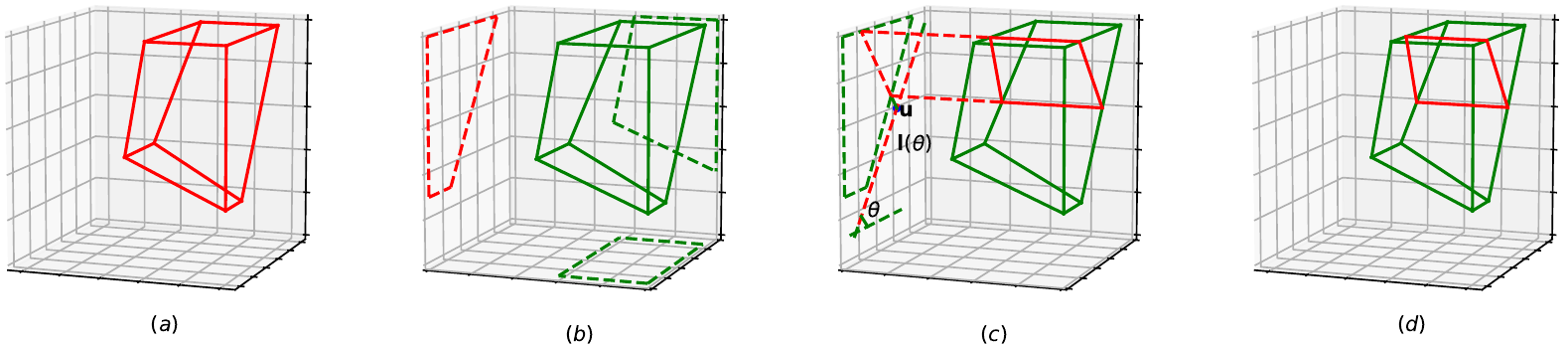}
\caption{{Left: a realisation of a 3-dimensional BSP-Tree Process with budget $\tau$. Each red-line constituted polygon denotes a new cutting hyperplane generated at a particular time. Right: visualization of generating a cutting hyperplane in a $3$-dimensional block $\square$ in Algorithm \ref{algo:BSP_partition_one_block}: (a) Visualize the block $\square$; (b) Select the dimensional pair; (c) Sample the direction $\theta$ and cut position $\pmb{u}$; and (d) Form a new cut in $\square$.}}
\label{fig:generativeprocesspartition}
\end{figure*}

%%%%%%%%%%%%%%%%%%%%%%%%%%%%%%%%%%%%%%%%%%%%%%
%%%%%%%%%%%%%%%%%%%%%%%%%%%%%%%%%%%%%%%%%%%%%%%%%
\section{The BSP-Tree Process}
%%%%%%%%%%%%%%%%%%%%%%%%%%%%%%%%%%%%%%%%%%%%%%%%%
%%%%%%%%%%%%%%%%%%%%%%%%%%%%%%%%%%%%%%%%%%%%%%

The BSP-Tree Process~\cite{pmlr-v84-fan18b,pmlr-v89-fan18a} is a time-indexed, right-continuous Markov jump process. At any point of the time line, the BSP-Tree Process takes the state of a hierarchical binary partition over the multi-dimensional space (Figure~\ref{fig:generativeprocesspartition}). Starting at time $0$ with the initial space (no cuts), new random oblique cuts are created over time. 
%Using blocks to denote the 
Each selected region for partition is named a block, where each cut recursively divides the block into two smaller ones. Given any two time points $0<t_1\le t_2$, the partition 
%$\boxplus_{t_2}$ 
at $t_2$ is a further refinement 
% in $\boxplus_{t_1}$) 
of the partition 
%$\boxplus_{t_1}$ 
at $t_1$ (e.g.~there might be additional cuts). 
% Note: I removed the $\boxplus$ notation here as it has not been introduced yet.

The distinguishing feature of the BSP-Tree Process is the oblique cut it uses to partition the space. Unlike its axis-aligned cut counterparts \cite{roy2009mondrian}, the BSP-Tree Process uses more than one dimension to form the cutting hyperplanes, and so is  a more efficient cutting strategy (i.e.~it uses fewer cuts to achieve similar performance).

We introduce some notation before formally describing the BSP-Tree Process. For a $d$-dimensional space $\square$~(which is assumed to be a convex polyhedron), we let $\mathcal{D}=\{(1, 2), (1, 3), \ldots, (d-1, d)\}$ denote the set of all  two dimensional pairs. For each element $(d_1, d_2)$ in $\mathcal{D}$, $\Pi_{d_1, d_2}(\square)$ denotes the projection of $\square$ onto the dimensions of $(d_1,d_2)$ (which becomes a 2-dimensional polygon), $L_{(d_1, d_2)}(\square)$ denotes the corresponding perimeter, $\pmb{l}_{\Pi_{d_1, d_2}(\square)}(\theta)$ lies in the 2D space of dimensions $(d_1, d_2)$ and represents the line segment of $\Pi_{d_1, d_2}(\square)$ in the direction of the angle $\theta$ (in the dimensions of $(d_1, d_2)$; Figure~\ref{fig:generativeprocesspartition}~(Right)), $\pmb{u}$ is the cutting position in the segment $\pmb{l}_{\Pi_{d_1, d_2}(\square)}(\theta)$, $c$ represents the cost of a cut in the time line and $C=(c, (d_1^*, d_2^*), \theta, \pmb{u})$ represents one cutting hyperplane with all these variables.

Algorithm~\ref{algo:BSP_partition total} describes generation of the recursive hierarchical partition under the BSP-Tree Process. Algorithm~\ref{algo:BSP_partition_one_block} generates an oblique cut in a particular block $\square$; the right part of Figure~\ref{fig:generativeprocesspartition} illustrates Lines 3--6 in Algorithm~\ref{algo:BSP_partition_one_block}. To generate an oblique cut in $\square$, we first generate the cost of the proposed cut from an Exponential distribution, parameterised by the sum of perimeters of all of the block's two-dimensional projections~(Line~$1$). If $c>\tau$ so the cost of the candidate cut exceeds the budget $\tau$, the cut is not generated and the current block $\square$ is returned as the final partition structure in that branch~(Line~$9$); otherwise we sample a 2-dimensional pair $(d_1^*, d_2^*)$ in proportion to $\{{L}_{(d_1, d_2)}(\square)\}_{(d_1,d_2)\in\mathcal{D}}$~(Line~$3$).  The cutting hyperplane will be generated in this dimensional pair, and is parallel to the other dimensions. Line~$4$ generates the angle $\theta^*$ with density proportional to $|\pmb{l}_{\Pi_{d_1^*, d_2^*}(\square)}(\theta)|$, and the cutting position $\pmb{u}$ uniformly in the segment $\pmb{l}_{\Pi_{d_1^*, d_2^*}(\square)}(\theta^*)$ for the pair $(d_1^*, d_2^*)$. The cutting hyperplane is then formed as $H\left((d_1^*, d_2^*), \theta^*, \pmb{u}\right)=\{\pmb{x}\in\square\vert([x_{d_1^*}, x_{d_2^*}]-\pmb{u})(1;\tan\theta^*)^{\top}=0\}$ and cuts the block $\square$ into two new sub-blocks~(Line~$5$). Line~$6$ recursively generates {\it independent} BSP partitions $\mathcal{B}', \mathcal{B}''$ on the new blocks $\square',\square''$. %Their allocated budget is the original budget minus the cost of the cut. Line~$9$: represents the node of the BSP-Tree in its hierarchical structure. 

One attractive property of BSP-Tree Process is that it is self-consistent~\cite{pmlr-v84-fan18b,pmlr-v89-fan18a}. This helps to provide an efficient convex hull representation of nodes in a tree when in a multi-dimensional space.
\begin{theorem} \label{fact_1}
(Self-Consistency) Given a partition sampled from the BSP-Tree Process on the convex domain $\square$, let $\triangle\subset\square$ be a convex subdomain of $\square$. Then, the restriction of this partition on $\triangle$ has the same distribution as
the restriction directly sampled 
from the BSP-Tree Process on $\triangle$. 
\end{theorem}
\begin{algorithm}[t]
\caption{\small BSP~($\square, \tau$)}
\label{algo:BSP_partition total}
{\small\begin{algorithmic}[1]
\STATE Call BlockCut($\square, \tau$)
\end{algorithmic}}
\end{algorithm}
\begin{algorithm}[t]
\caption{\small BlockCut~($\square, \tau$)}
\label{algo:BSP_partition_one_block}
{\small \begin{algorithmic}[1]
%  \REQUIRE Block $\square$, budget $\tau$
%  \ENSURE A set of sub-blocks $\{\square^{(k)}\}_k$
\STATE Sample cost {\small $c\sim\text{Exp}(\sum_{(d_1,d_2)\in\mathcal{D}}L_{(d_1, d_2)}(\square))$}
\IF{$c<\tau$}
\STATE  $(d_1^*, d_2^*)\sim p(d_1, d_2)\propto{L}_{(d_1, d_2)}(\square), {(d_1,d_2)\in\mathcal{D}}$
\STATE $\theta^*\sim p(\theta)\propto |\pmb{l}_{\Pi_{d_1^*, d_2^*}(\square)}(\theta)|, \theta\in(0, \pi]$, Sample $\pmb{u}$ uniformly on $\pmb{l}_{\Pi_{d_1^*, d_2^*}(\square)}(\theta^*)$
%\STATE $\pmb{u}\sim\text{Unif}(0, |\pmb{l}_{\Pi_{d_1^*, d_2^*}(\square)}(\theta)|]$
\STATE Form cutting hyperplane $H\left((d_1^*, d_2^*), \theta^*, \pmb{u}\right)$ and cut $\square$ into two sub-blocks $\square', \square''$
\STATE $\mathcal{B}'=$ BlockCut($\square', \tau-c$), $\mathcal{B}''=$ BlockCut($\square'', \tau-c$)
\STATE Return $\{\square, C=\{c, (d_1^*, d_2^*), \theta^*, \pmb{u}\}, \mathcal{B}', \mathcal{B}''\}$
\ELSE
\STATE Return $\{\square, \emptyset, \emptyset, \emptyset\}$
\ENDIF
\end{algorithmic}}
\end{algorithm}
%\begin{theorem} \label{fact_1}
%(Self-Consistency) Let $\boxplus\sim\text{BSP}(\square, \tau)$ be a Binary Space Partition from the BSP-Tree Process on the convex domain of $\square$, and $\triangle\subset\square$ be a sub-convex domain of $\square$. Then, the restriction of $\boxplus$ on $\triangle$ has the same distribution as
%%the directly sample on 
%a partition on $\triangle$ directly sampled 
%from the BSP-Tree Process. 
%\end{theorem}
% It will also be mentioned and used in proving its consistency result. 

%%%%%%%%%%%%%%%%%%%%%%%%%%%%%%%%%%%%%%%%%%%%%%
%%%%%%%%%%%%%%%%%%%%%%%%%%%%%%%%%%%%%%%%%%%%%%
\section{Online BSP-Forests}
%%%%%%%%%%%%%%%%%%%%%%%%%%%%%%%%%%%%%%%%%%%%%%
%%%%%%%%%%%%%%%%%%%%%%%%%%%%%%%%%%%%%%%%%%%%%%

%In online BSP Forest, 
We have an observed set of $N$ labelled datapoints $\{(\pmb{x}_n, y_n)\}_{n=1}^N\in\mathbb{R}^d\times \mathbb{R}$ which arrive over time, where $\pmb{x}_i$ is  a $d$-dimensional feature vector and $y_i$ is the corresponding label. We have an additional set of feature data for predictive testing.
The goal is to predict the unknown labels of the testing data, based on their feature data and the observed  training data $\{(\pmb{x}_n, y_n)\}_{n=1}^N$. 

Similar to the standard random forest~\cite{breiman2000some,biau2008consistency}, which assumes the partition is generated independently of the data labels, the online BSP-Forest does not use the data labels, but considers the partition structure within the convex hulls spanned by the arrival feature data $\pmb{x}$. Each node in the BSP-Tree records two items: the convex hull that covers its constituent datapoints, and the possible hyperplane cutting this hull. A nested set $\mathcal{V}^{(0)}$ may be used to represent the tree structure of these nodes. $\forall l\in\{0, \ldots, L\}$, $\mathcal{V}^{(l)}$ is defined as $\mathcal{V}^{(l)}=\{\Diamond^{(l)},C^{(l)}, \mathcal{V}^{(l+1)}_{\text{left}}, \mathcal{V}^{(l+1)}_{\text{right}}\}$ , where $\Diamond^{(l)}$ refers to the given convex hull, $C^{(l)}$ is $\Diamond^{(l)}$'s cutting hyperplane, and $\mathcal{V}^{(l+1)}_{\text{left}}$ ($\mathcal{V}^{(l+1)}_{\text{right}}$) represents the left (right) child node of $\mathcal{V}^{(l)}$. Since $\mathcal{V}^{(0)}$ is generated depending on the feature points $\pmb{x}_{1:N}$, it can be generated as $\mathcal{V}^{(0)}\sim\text{BSP}_{\pmb{x}}(\pmb{x}_{1:N}, \tau_{1:N})$, where $\tau_{1:N}$ is an increasing budget sequence. 

Further, the BSP-Forest, an ensemble of $M$ independent BSP-Trees $\{\mathcal{V}^{(0)}(m)\}_{m=1}^M$ sampled over the feature space, may be used for regression/classification tasks. 
In this way, the predictive label distribution on any feature data $\pmb{x}_n$ is $g(\pmb{x}_n)=\frac{1}{M}\sum_{m=1}^Mp(y|\pmb{x}_n, \mathcal{V}^{(0)}(m))$, where $p(y|\pmb{x}_n, \mathcal{V}^{(0)}(m))$ is the predictive distribution of $y$ under the $m$-th BSP-Tree. That is, the predictive distribution is a finite sample approximation to
$\mathbb{E}_{\mathcal{V}^{(0)}\sim \text{BSP}_{\pmb{x}}(\pmb{x}_{1:N}, \tau_{1:N})}[p(y|\pmb{x}, \mathcal{V}^{(0)})]$, which becomes more accurate for a larger $M$.

\begin{algorithm}[t]
\caption{\small $\text{oBSP}(\pmb{x}_{1:N}, \tau_{1:N})$ }
\label{algo:OnlineBSP}
{\small \begin{algorithmic}[1]
  \STATE Initialise $\mathcal{V}^{(0)}(m)=\{\emptyset, \emptyset, \emptyset, \emptyset\}, \forall m\in\{1, \ldots, M\}$ 
  \FOR{$n=1, \ldots, N$}
%  \FOR{$m=1, \ldots, M$}
  \STATE $\mathcal{V}^{(0)}(m)= \text{cBSP}(\mathcal{V}^{(0)}(m), \pmb{x}_n, \tau_n)$, $\forall m\in\{1,\ldots, M\}$
%  \ENDFOR
  \ENDFOR
\end{algorithmic}}
\end{algorithm}
\begin{algorithm}[t]
%\caption{ExtenBlock ($\pmb{H}, \pmb{C}, \tau, \pmb{x}_i$)}
\caption{\small $\text{cBSP}(\mathcal{V}^{(l)}=\{\Diamond^{(l)},C^{(l)}, \mathcal{V}^{(l+1)}_{\text{left}}, \mathcal{V}^{(l+1)}_{\text{right}}\}, \pmb{x}, \tau)$ }
\label{algo:extensionBlock}
{\small\begin{algorithmic}[1]
%  \REQUIRE Block $\square$, current accumulated cost $c$, budget $\tau$, new datapoint $\pmb{x_n}$
%\STATE Include $\pmb{x}_i$ to form ${\square'} = {\square}\cup\pmb{x}_i$  
\STATE Define $\Diamond'$ as convex hull covering both $\Diamond^{(l)}$ and $\pmb{x}$.
%\IF{$S_{\square'}\le 3$ }
\IF{${\Diamond'}$ contains $\le 3$ datapoints}
\STATE Replace $\Diamond^{(l)}$ with $\Diamond'$ in $\mathcal{V}^{(l)}$ and return $\mathcal{V}^{(l)}$
\ELSIF{$\mathcal{V}^{(l)} \text{ is a leaf node}$}
\STATE Return $\text{HullCut}(\Diamond^{(l)}, \tau)$   {\color{white!50!black}// $\Diamond^{(l)}$ is the convex hull in $\mathcal{V}^{(l)}$}
\ENDIF
\STATE Sample a cost variable $c'\sim\text{Exp}(\lambda)$, where $\lambda=\sum_{(d_1,d_2)\in{\mathcal D}}[L_{(d_1, d_2)}(\Diamond')-L_{(d_1, d_2)}(\Diamond^{(l)})]$
\IF{$c'<c^{(l)}$} {\color{white!50!black}// $c^{(l)}$ is the cut cost in $C^{(l)}$}
\STATE Generate cut $C'$ in $\Diamond'$ , with $C'$ not crossing into $\Diamond^{(l)}$
\STATE Increase all the level index ($l$) in $\mathcal{V}^{(l)}$ by $1$, return $\mathcal{V}^{(l)}=\{\Diamond', C', \mathcal{V}^{(l+1)}, \{\pmb{x}, \emptyset, \emptyset, \emptyset\}\}$ 
\ELSE
\IF{$\pmb{x}\notin \Diamond^{(l)}$} 
\STATE Extend current cut $C^{(l)}$ to $\Diamond'$ and use $\Diamond'$ to replace $\Diamond^{(l)}$
\ENDIF
\STATE  cBSP($\mathcal{V}^{(l)}_{\text{left}}, \pmb{x}, \tau-c^{(l)}$) {\color{white!50!black}// assume that $\pmb{x}$ belongs to the left side of $C^{(l)}$}
\ENDIF
\end{algorithmic}}
\end{algorithm}
\begin{algorithm}[t]
\caption{\small HullCut~($\Diamond, \tau$)} 
\label{algo:CutConvexHull}
{\small \begin{algorithmic}[1]
\FOR{$(d_1, d_2) \in \mathcal{D}$ }
    \STATE Extract dimensions $(d_1, d_2)$ of $\{\pmb{x}_i\}_{i}\in \Diamond$  to obtain $\{x_{i,d_1}, x_{i,d_2}\}_{i}$
    \STATE Calculate the convex hull on $\{x_{i,d_1}, x_{i,d_2}\}_{i}$, denoted by $\Diamond_{(d_1,d_2)}$
\ENDFOR
\STATE Sample cost  { $c\sim\mbox{Exp}(\sum_{(d_1,d_2)\in\mathcal{D}}L(\Diamond_{(d_1,d_2)}))$}
\IF{$c<\tau$}
\STATE Sample $(d^*_1,d^*_2)$ from $\mathcal{D}$ in proportion to $\{L(\Diamond_{(1,2)}),\ldots, L(\Diamond_{(d-1,d)})\}$
\STATE Sample $\theta, \pmb{u}$ on the projection of the convex hull $\Diamond_{(d^*_1,d^*_2)}$ 
\STATE Use the new cutting hyperplane to generate two new convex hulls $\Diamond', \Diamond''$
\STATE $\mathcal{V}'=\text{HullCut}(\Diamond', \tau-c$), $\mathcal{V}''=\text{HullCut}(\Diamond'', \tau-c$)
\STATE Return $\{\Diamond, C = \{c, (d_1^*, d_2^*), \theta, \pmb{u}\}, \mathcal{V}', \mathcal{V}''\}$
\ELSE
\STATE Return $\{\Diamond, \emptyset, \emptyset, \emptyset\}$
\ENDIF
\end{algorithmic}}
\end{algorithm}

\begin{figure*}[t]
\centering
\includegraphics[width = 0.85 \textwidth]{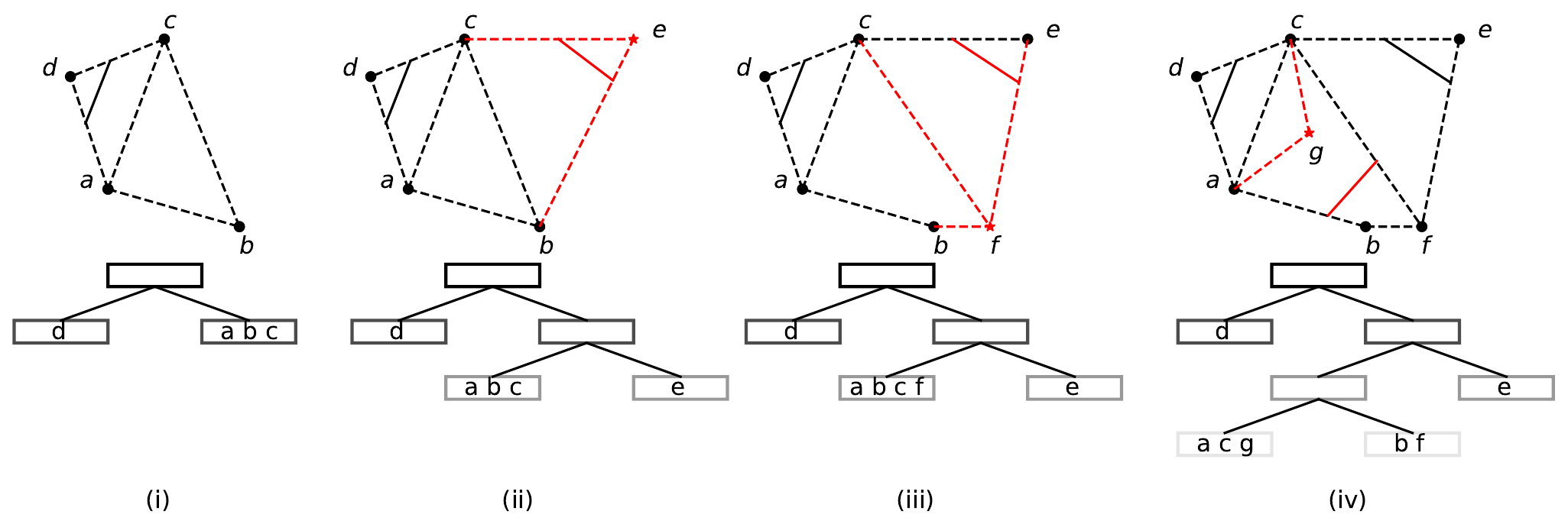}
\caption{{Steps (i)--(iv): A $2$-dimensional illustration of a sequence of online BSP-Tree updates on arrival of new datapoints (denoted by {$\star$}). Dashed lines indicate edges of convex hulls and solid lines represent cutting lines. Black and red colouring respectively denotes existing and newly active components at the current step. The tree structure (bottom figure) can be identified through the subset relations between the convex hulls.}}
\label{fig:Splitting_rule}
\end{figure*}

Algorithm~\ref{algo:OnlineBSP}~(oBSP, along with Algorithms \ref{algo:extensionBlock} and \ref{algo:CutConvexHull}) specifies the strategy of building the online BSP-Forest in detail. Starting from a set of empty nodes, the online BSP-Forest updates each BSP-Tree once new training data $\pmb{x}_{n+1}$ arrives. From the BSP-Tree constructed from the previous $n$ datapoints,
this procedure then defines a {\it conditional} BSP-Tree Process~(cBSP) such that,
following the distributional invariance self-consistency property  between the domain and its subdomain, and the Markov property in the time line, the new partition $\mathcal{V}^{'(0)}$ based on up to $n+1$ datapoints is constructed as
\begin{align*}
\mathcal{V}^{(0)} \sim \text{BSP}_{\pmb{x}}(\pmb{x}_{1:n}, \tau_{1:n}),\mathcal{V}^{'(0)} \sim \text{cBSP}(\mathcal{V}^{(0)},\pmb{x}_{n+1},\tau_{n+1}).
\end{align*}
It then follows that the new partition $\mathcal{V}^{'(0)}$ has the same distribution as
\begin{align*}
\mathcal{V}^{'(0)} \sim \text{BSP}_{\pmb{x}}(\pmb{x}_{1:n+1}, \tau_{1:n+1}).
\end{align*}
That is, the distribution of an online BSP-Tree trained (using the cBSP) on sequentially arriving data, is the same as that of a BSP-Tree~($\text{BSP}_{\pmb{x}}$) directly trained on the full dataset. Different shuffling arrangement on the datapoints would not affect the partition result.

\paragraph{Conditional BSP-Tree Process~(cBSP):} Algorithm~\ref{algo:extensionBlock} describes the procedural detail of the cBSP when incorporating a new datapoint $\pmb{x}$. Based on the location of $\pmb{x}$ and the (possibly) generated cut cost, each BSP-Tree might update the tree structure,
%~(Algorithm~\ref{algo:extensionBlock}, lines $5, 10, 15$), 
enlarge the hull coverage,
%~(lines $1, 3, 13$), 
or implement a new cut in the tree nodes.
%~(lines $5,9, 15$). \textcolor{red}{[Are these lines numbers correct? They don't seem to match Alg 4?]} 
Figure~\ref{fig:Splitting_rule} illustrates each of these cases:
\begin{itemize}
\item[({\it i})] Points $a, b,c,d$ form a $2$-level BSP-Tree~(Lines~1--6, Algorithm~\ref{algo:extensionBlock}), where $a,b,c$ are in the right level-$1$ convex hull and $d$ is in the left one. 
\item[({\it ii})] A new point $e$ arrives outside the level-$0$ convex hull. A new cost value $c'$ is generated (Line~7, Algorithm~\ref{algo:extensionBlock}). This new cost $c'$ is compared with the cost $c$ associated with the cut-line in the level-$0$ convex hull. If $c'<c$ (Line~8, Algorithm~\ref{algo:extensionBlock}), a cut-line not crossing into the level-$0$ convex hull is generated (see Section 1 of the supplementary material), and a new level-$0$ convex hull is formed from the previous level-$0$ hull and the point $e$. All previously existing level-$l$ convex hulls in this branch are demoted to level-($l+1$) as a consequence of creating the new right level-$1$ convex hull. 
\item[({\it iii})] A new point $f$ arrives outside the level-$0$ convex hull and the new generated cost value $c'$ is larger than the current cost value $c$. The cut-line (red solid line) allocates the new point to the rightmost level-$2$ convex hull. Since this is a leaf node (no cut-lines inside), it forms a new level-$2$ convex hull including the new point (Lines~13 and 15, Algorithm~\ref{algo:extensionBlock}).
\item[({\it iv})] A new point $g$ arrives inside the level-$0$ convex hull. Different levels of cutting lines are repeatedly compared to $g$ until $g$ is compared with a leaf node~(Line~15, Algorithm~\ref{algo:extensionBlock}). Since a larger budget $\tau_7$~(i.e. the budget for $7$ datapoints) is provided, new cuts may still be generated within the leaf node.
\end{itemize}

{It is noted that Lines $2\sim3$ in Algorithm~\ref{algo:extensionBlock} restrict a block to be cut only if it contains more than $3$ datapoints. If we have set an overly-large budget sequence, which prefers to generate small and trivial blocks, this restriction can help the algorithm avoid the overfitting issue. }

\paragraph{Benefits of a convex hull representation:} %As long as a $d$-dimensional convex hull $\square$ contains $\geq 3$ non-colinear datapoints, it is possible to form a convex hull in their projection on any two sub-dimensions. For such convex hulls, implementing (axis-aligned) cutting hyperplanes directly on these polyhedrons is quite inefficient, and this was one of the motivating factors for the development of the BSP-Tree \cite{pmlr-v84-fan18b}. Figure~\ref{fig:mt_vs_convex_hull} visually shows that the BSP-Tree generates smaller convex hulls than the Mondrian-Tree.

Cutting the full $d$-dimensional polyhedron generated from a series of cutting hyperplanes on the original domain is a challenging  task. This is because 
%(1)
completely indexing of this polyhedron requires extensive geometric manipulation, including calculating the intersection of multiple hyperplanes, determining if two hyperplanes intersect in the presence of  other hyperplanes, and listing all  vertices of the polyhedron.

Instead, implementing the hyperplane cut on the convex hull greatly simplifies these issues. The projection of the convex hull on any two dimensions can be simply obtained by extracting the elements of $\{\pmb{x}_i\}_{i=1}^n$ in these dimensions, and then using a conventional convex hull detection algorithm~(such as the Graham Scan algorithm~\cite{Graham1972} with a computational complexity of $\mathcal{O}(n\log n)$). Algorithm~\ref{algo:CutConvexHull} describes the way of generating the hyperplane cut in a convex hull.

The self-consistency property~(Theorem~\ref{fact_1}) of the 
BSP-Tree Process enables us conveniently and immediately have that: 
\begin{corollary}
The hyperplane restricted to the convex hull has the same distribution as if it was first generated on the ``full'' polyhedron, and then restricted to the convex hull. 
\end{corollary}
That is, the partitions of the BSP-Tree Process on the full space and on the data spanned convex hull will produce identical inferential outcomes.

{\paragraph{Computational complexity:}
In the optimal case and assuming the BSP-Tree formed by the streaming data is a balanced tree (which keeps the height of trees in a complexity of $\log N$), the computational complexity of a single BSP-Tree is $\mathcal{O}(N\log N\cdot d^2)$. When it is of the same magnitude as the  $\mathcal{O}(N\log N\cdot d)$ of the Mondrian Forest~\cite{LakRoyTeh2014a,lakshminarayanan2016mondrian} ~(also in the optimal case) in terms of $N$, the additional $\mathcal{O}(d)$ scaling factor is the price of more efficient cuts. In this perspective, this online BSP-Forest might be more suitable for low-dimensional regression/classification problems. In practice, we may alleviate this additional computational cost by using less hierarchical structure (through settings of smaller budget $\tau$, see experiments in Section~\ref{exp:cut_efficiency}). The convex hull calculation is $\mathcal{O}(N\log N)$, which does not influence the general computational complexity. 

In the worst case, when the height of the BSP-Tree equals to the number of datapoints~($N$) and each level has one split only, the computational complexity of the online BSP-Forest turns into $\mathcal{O}(N^2d^2)$~(which is the same as the Mondrian Forest in terms of $N$). This complexity is still smaller than that of the batch version of a Random Forest~($\mathcal{O}(d\cdot N^2\log N)$~\cite{LakRoyTeh2014a,lakshminarayanan2016mondrian}) in terms of $N$. It is also noted that the online BSP-Forest can be easily made to run in parallel as its component trees are updated in an independent manner.

When a new sample arrives, the computational complexity depends on the existing tree structure: if the BSP-tree is a balanced tree, the complexity is $\mathcal{O}(d^2\log N)$; if the height of the BSP-tree equals to the number of existing samples~(i.e., the worst case), the complexity is $\mathcal{O}(d^2N)$.

In the memory usage part, the online BSP-Forest requires the same amount of $\mathcal{O}(Nd)$ as the Mondrian Forest. That is to say, each leaf node of the BSP-Tree would store all its belonging datapoints for future splitting. }

\paragraph{Empirical label distribution:} For the label distribution  $p(y|\pmb{x}, \mathcal{V}^{(0)})$, we use the empirical distribution of the labels from the training data to represent the label distribution on each of the leaf nodes. For classification tasks, we use a majority vote over the label predictions of each tree in the online BSP-Forest; for regression tasks, we use the mean value of the trees' label predictions at the given point. According to~\cite{breiman2000some,biau2008consistency,consistencyMondrianforest}, this simple empirical approach usually performs competitively compared to constructing complex Bayesian models on the label distribution. 

\paragraph{Universal consistency:}
As a larger value of the budget $\tau_n$ indicates a deeper tree structure and more complex partitions, the increasing budget sequence $\tau_{1:N}(\tau_{n+1}\ge\tau_{n})$ indicates that the online BSP-Forest will produce more refined partitions with increasing amounts of data. This is consistent with the intuition behind, and a central tenet of, Bayesian nonparametrics. Similar observations~(for Mondrian Forest only) has also been obtained in
\cite{consistencyMondrianforest}. Further, it shows that such an increasing budget sequence is the key to guarantee the algorithm's performance converge to the optimal for all distributions over the labels and features.

 More formally, let $\ell_n=P(g_n(\pmb{x})\neq y|\{(\pmb{x}_i, y_i)\}_{i=1}^n)$ denote the error probability of these tree-structured online algorithm $g_n(\cdot)$. Using $\ell^*=\min_{g_n} \ell_n$ to denote the Bayesian error rate, which is generated by the optimal classifier. As the increasing budget with $n$ results in finer partitions (i.e. more blocks and ``smaller'' block sizes), these online algorithms ensure there are a sufficiently large number of datapoints in each block, and its performance approaches the Bayesian error. 

In order to discuss this ``universal consistency'' in the online BSP-Forest setting, we first investigate particular properties of an oblique line slice of the BSP-Tree Process.
\begin{lemma} \label{lemma:dim_1_poisson_process}
(Oblique line slice) For any oblique line that crosses into the domain of a BSP-Tree Process with budget $\tau$, its intersection points with the partition forms a homogeneous Poisson Process with intensity $2\tau$.
\end{lemma}

%Based on Theorem~\ref{fact_1} and Lemma~\ref{lemma:dim_1_poisson_process}, we can obtain similar results to the budget increasing Mondrian-Forest~\cite{consistency_Mondrian_forest} regarding the expected partition diameter and the expected number of cuts for the online BSP-Tree.
%\begin{lemma} \label{lemma:cell_diameter}
%(Cell diameter) Let $\pmb{x}\in[0, 1]^d$, and let $D(\pmb{x})$ be the ${\it L}^2$-diameter of the cell containing $\pmb{x}$ in the BSP-Tree partition 
%~(with $\tau/2$ budget). If $\tau\to\infty$, then $D_{\tau}(\pmb{x})\to 0$ in probability. More precisely, for every $\delta, \tau>0$, we have
%\begin{align*}
%\mathbb{P}(D_{\tau}(\pmb{x})\ge\delta)\le d(1+\frac{\tau\delta}{\sqrt{d}})\exp(-\frac{\tau\delta}{\sqrt{d}}),\quad\text{ and }\quad \mathbb{E}[D_{\tau}(\pmb{x})^2]\le\frac{4d}{\tau^2}.
%\end{align*}
%\end{lemma}
%
%\begin{lemma} \label{lemma:expected_k}
%If $K_{\lambda}$ denotes the number of cuts in the partition sampled from the BSP-Tree Process, we have $\mathbb{E}[K_{\lambda}]\le\left(e(1+\lambda)\right)^d$.  
%%\textcolor{red}{checking: $e\approx 2.71$, or typo?}
%\end{lemma}
%While Lemma~\ref{lemma:cell_diameter} ensures the diameter of each cell converges to $0$ as $n\to\infty$, Lemma~\ref{lemma:expected_k} ensures the upper bound on the number of cells in $[0,1]^d$, which can further guarantee enough points in each block when $n\to\infty$~(\cite{consistency_Mondrian_forest}). 
Lemma~\ref{lemma:dim_1_poisson_process} can enable us to describe basic characteristics~(e.g. size, number) of the blocks of the BSP-Tree. Based on Lemma~\ref{lemma:dim_1_poisson_process}, the theoretical work of \cite{consistencyMondrianforest} and Theorem 6.1 of~\cite{devroye2013probabilistic} and set restrictions on the budget sequence, we can obtain the resulting universal consistency for the online BSP-Forest as:
\begin{theorem}
If $\lim_{n\to\infty}\tau_n\to\infty$ and $\lim_{n\to\infty}\frac{(\tau_n)^d}{n}\to 0$, then for classification tasks, we have $\lim_{n\to\infty}\mathbb{E}[\ell_n]\to \ell^*$.
\end{theorem}
This universal consistency can be applied on any format of dataset as the result does not restrict the distribution of $\pmb{x}$ or the prediction function $g_n(\cdot)$. 

Using the same proof techniques of~\cite{consistencyMondrianforest}, we can obtain minimax convergence rate for the regression trees task as~(the result for the classification trees can be obtained in a similar way):
\begin{theorem}
Suppose that the label $y$ is generated by a Lipschitz function $g(\cdot): [0, 1]^d\to \mathbb{R}$ on the cubic space $[0, 1]^d$. Let $\widehat{g}_n(\cdot)$ be an online BSP-Forest algorithm and that the budget sequence satisfies $\tau_n=\mathcal{O}(n^{1/(d+2)})$. Then, the following upper bound holds:
\begin{align}
\mathbb{E}_{\pmb{x}}\left[(g(\pmb{x})-\widehat{g}_n(\pmb{x}))^2\right]=\mathcal{O}(n^{-2/(d+2)})
\end{align}
for sufficiently large $n$, which corresponds to the minimax rate over the set of Lipschitz functions.
\end{theorem}
This result can be used for the guide of appropriately choosing the values on the budget sequence.
%{\color{green}
%We build up our theoretical analysis based on the work of \cite{consistency_Mondrian_forest}. Let $\ell_n=P(g_n(x)\neq Y)$ denote the error probability of the online BSP-Forest classifier, where $g_n(\cdot)$ is the prediction function provided by the online BSP-Forest. Denote by $\ell^*$ the Bayesian error rate. As the increasing budget with $n$ results in finer partitions (i.e. ``smaller'' block sizes), the online BSP-Forest ensures there are a sufficiently large number of datapoints in each block, and its performance approaches the Bayesian error. More formally, combining Lemma~\ref{lemma:dim_1_poisson_process} above, Lemmas~$1$ and $2$ from~\cite{consistency_Mondrian_forest} and Theorem 6.1 of~\cite{devroye2013probabilistic}, we can obtain the resulting universal consistency of the online BSP-Forest for classification tasks:
%\begin{theorem}
%If $\lim_{n\to\infty}\tau_n\to\infty$ and $\lim_{n\to\infty}\frac{(\tau_n)^d}{n}\to 0$, then for classification tasks, we have $\lim_{n\to\infty}\mathbb{E}[\ell_n]\to \ell^*$.
%\end{theorem}
%}

%%%%%%%%%%%%%%%%%%%%%%%%%%%%%%%%%%%%%%%%%%%%%%
%%%%%%%%%%%%%%%%%%%%%%%%%%%%%%%%%%%%%%%%%%%%%%
\section{Related Work}
%%%%%%%%%%%%%%%%%%%%%%%%%%%%%%%%%%%%%%%%%%%%%%
%%%%%%%%%%%%%%%%%%%%%%%%%%%%%%%%%%%%%%%%%%%%%%
Since its introduction in the early 2000s, the random forest~\cite{breiman2001random,breiman2000some} has become a state-of-the-art method for typical classification and regression tasks. The literature is too vast to provide a comprehensive but compact survey, and so we focus on its online versions. For the frequentist-styled online random forest algorithms, they were developed by~\cite{consistency_online_BF,On_Line_RF,domingos2000mining}. The first two algorithms start from empty trees, and then grow the tree with more data based on evaluating a score function for each potential split position. As this score function relates to training performance, the node splitting procedure is label dependent. However, memory usage is inefficient in these two algorithms as the scores must be computed and stored for each potential split position. The third algorithm proposes to build online decision trees using constant memory and constant time per datapoint.  These partitions are inefficient as only one feature is involved in the split procedure.
%
%In these works, each tree is starting from an empty tree and grows when more data arrives. For the node splitting procedure, there is a score function for each potential split positions. The score function relates to the training performance and in this sense, the node splitting procedure is data dependent. Thus, the memory usage is inefficient as they need to store the score values for all the potential split positions.

The Purely Random Forest~(PRF) algorithm~\cite{genuer2012variance, arlot2014analysis} assumes  tree generation is independent of data labels. When the split position is random, which is similar to the Mondrian Forest and the online BSP-Forest, the corresponding distribution on the resulting partition is not self-consistent. Further, its batch learning framework is not amenable to large-scale learning. \cite{genuer2012variance} proves that the PRF can achieve the minimax rate for estimating Lipschitz functions in the $1$-dimensional case for single trees. \cite{arlot2014analysis} extends the analysis to forest settings and shows an improved convergence rate for smooth regression functions. 

The Mondrian Forest~\cite{LakRoyTeh2014a,lakshminarayanan2016mondrian} is the closest related method to the online BSP-Forest, as it uses the Mondrian process~\cite{roy2007learning,roy2009mondrian,roy2011thesis} to place a probability distribution over all the $k$d-tree-based partitions of the domain. In regularising the Mondrian-Forest to be consistent, \cite{consistencyMondrianforest} sets the budget parameter to increase with the amount of data, and it then achieves the minimax rate in multi-dimensional space for single decision trees. \cite{mourtada2018minimax} displays the advantage of Forest settings by showing improved convergence results.

\begin{figure*}[t]
\centering
\includegraphics[width =  1 \textwidth]{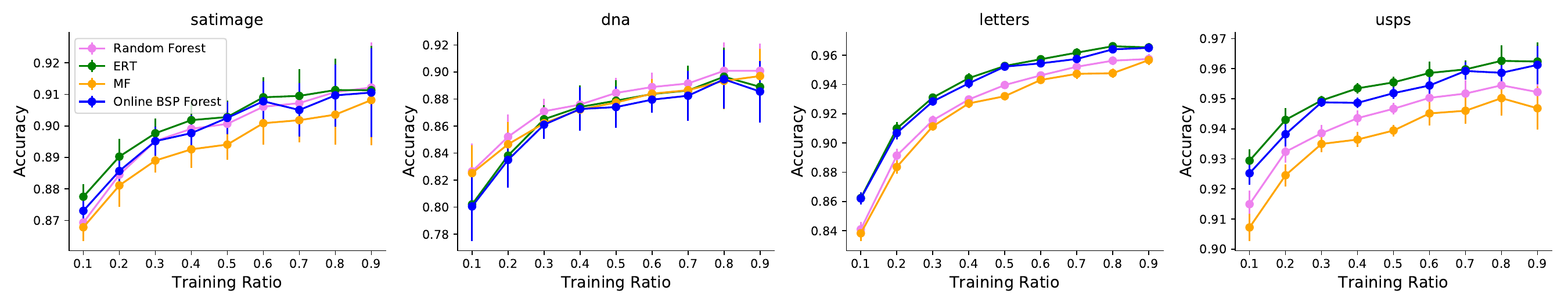}
\caption{{Classification accuracy~($\pm$ 1.96 standard error) for the satimage, dna, letters and usps datasets. The $x$-axis indicates to the ratio of training:testing data.
%, while the $y$-axis is the accuracy on the test data.  
}}
\label{fig:real_data_classifer}
\end{figure*}

In its favour, and in contrast to the above methods, the online BSP-Forest uses more than one dimension to implement node splits, and is accordingly more efficient. Further, the arrival of new data is dealt with by considering both the expansion of the feature space (i.e.~the convex hull representation of tree nodes) and expansion of the budget line (i.e.~an increasing budget sequence). This ensures the online BSP-Forest is represented efficiently and theoretically guaranteeing universal consistency.

%%%%%%%%%%%%%%%%%%%%%%%%%%%%%%%%%%%%%%%%%%%%%%
%%%%%%%%%%%%%%%%%%%%%%%%%%%%%%%%%%%%%%%%%%%%%%
\section{Experiments}
%%%%%%%%%%%%%%%%%%%%%%%%%%%%%%%%%%%%%%%%%%%%%%
%%%%%%%%%%%%%%%%%%%%%%%%%%%%%%%%%%%%%%%%%%%%%%
\label{section:results}

We examine the performance of the online BSP-Forest in regression and classification tasks. Unless specified, in each analysis, we re-scale the data features to the domain $[0, 1]^d$ space, set the parameter of the Exponential distribution of cut cost to $L(\square)/2$~(half of the perimeter of the ($2$-dimensional) block $\square$), and specify the budget sequence  as $\tau_n=n^{1/(d+2)}(\forall n\in\{1,\ldots,N\})$. Through these experiments, we show that (1) the online-BSPF performs better than other online algorithms and at the same time, it keeps competitive to the batch trained algorithms in regression and classification tasks; (2) the online-BSPF produce more efficient partitions than the Mondrian Forest algorithm.

\begin{figure}[t]
\centering
\includegraphics[width =  0.45 \textwidth]{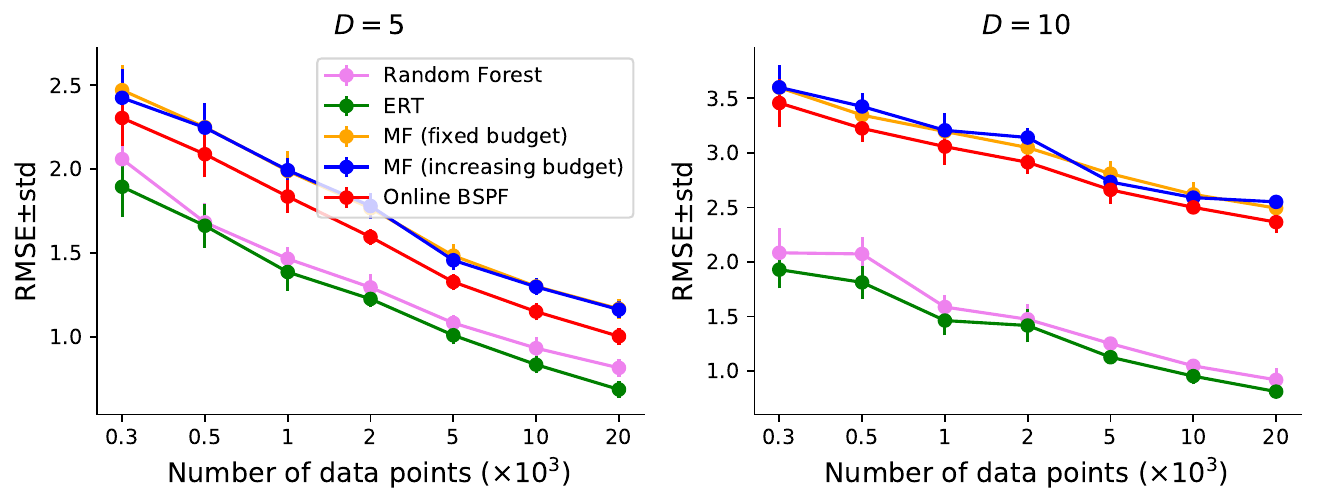}
\caption{{RMSE~($\pm$ 1.96 standard error) on Friedman's test function, as a function of the number of training data ($x$-axes) and dimension of the features $\pmb{x}$.}}
\label{fig:Friedman_result}
\end{figure}
The online BSP-Forest is compared with several state-of-the-art methods: (1) a random forest (RF)~\cite{breiman2001random}; (2) Extremely Randomized Trees (ERT)~\cite{geurts2006extremely}; (3) a Mondrian Forest (MF) with fixed budget~\cite{LakRoyTeh2014a, lakshminarayanan2016mondrian}; and (4) a Mondrian Forest with increasing budget~\cite{consistencyMondrianforest}. The number of trees is fixed to $M=100$ to avoid high computational costs. For the RF and ERT, we use the implementations in the {\tt scikit-learn} toolbox~\cite{scikit-learn}. For the MF with infinite budget, we use the {\tt scikit-garden}  implementation.
 % for MF with finite budget, we implement the method to the best of our understanding. 
For the parameter tuning in the RF and ERT, we use the {\tt GridSearchCV} package in the {\tt scikit-learn} toolbox and focus on tuning the features of ``{\tt max$\_$features}'', which concerns the number of features considered when selecting the best split. Each test case was replicated $16$ times, and we report summary means and $1.96\times$ standard errors for each measure.

\begin{table*}[t]
\caption{{\small Performance comparison on Apartment data and Airline data (RMSE$\pm$1.96 standard error)}}\label{table_4}
\centering
{\small\begin{tabular}{c|ccccc}
  \hline
{Dataset} & RF & ERT & MF (fixed $\tau$) & MF (increasing $\tau_n$) & { online BSP-Forest}
\\  \hline
Apartment & $\pmb{0.526\pm 0.02}$& $0.531\pm 0.02$ & $0.545\pm 0.04$& $0.541\pm 0.03$& $0.536\pm 0.03$ \\
\hline
Airline ($100K$) & $\pmb{32.2\pm0.7}$ & $34.1\pm1.0$ & $36.5\pm0.5$ & $35.1\pm1.2$ & $35.9\pm0.8$ \\
\hline
Airline ($400K$) & $33.7\pm0.4$ & $\pmb{33.5\pm0.9}$ & $35.8\pm1.2$ & $36.2\pm0.9$ & $35.2\pm1.3$ \\
\hline
Airline ($1M$) & $34.1\pm0.9$ & $\pmb{33.5\pm0.6}$ & $37.3\pm1.1$ & $36.5\pm0.6$ & {$35.9\pm1.1$} \\
\hline
\end{tabular}}
\end{table*}

\begin{figure*}[t]
\centering
\includegraphics[width =  0.48 \textwidth]{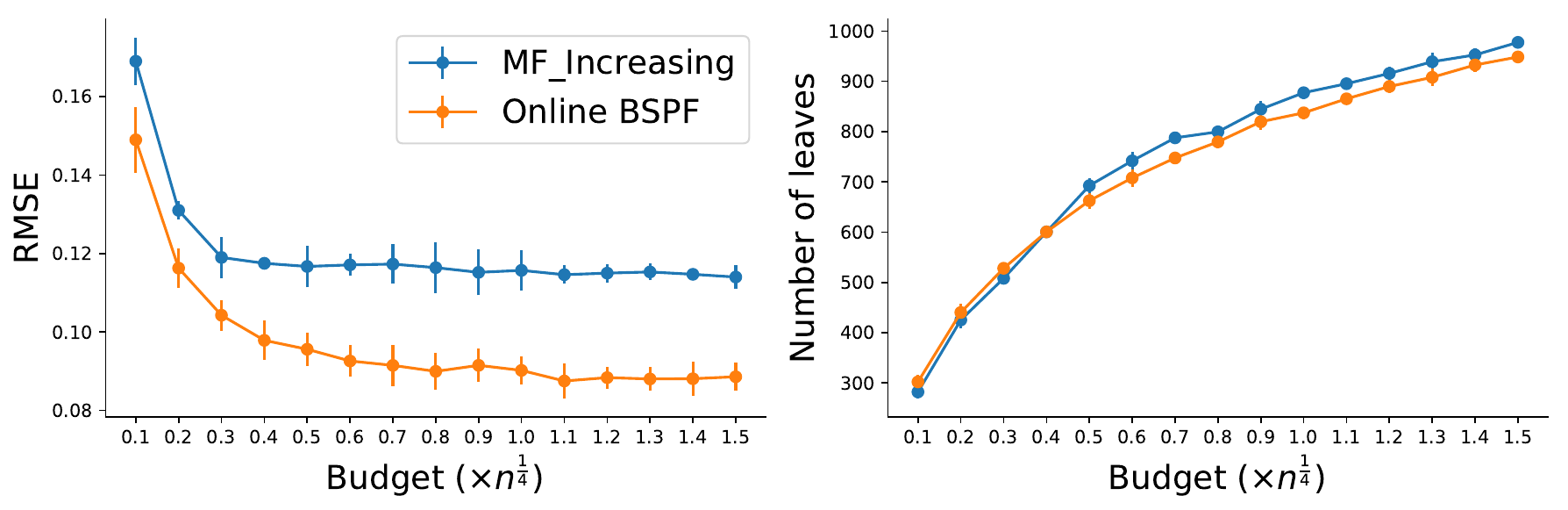}
\includegraphics[width =  0.48 \textwidth]{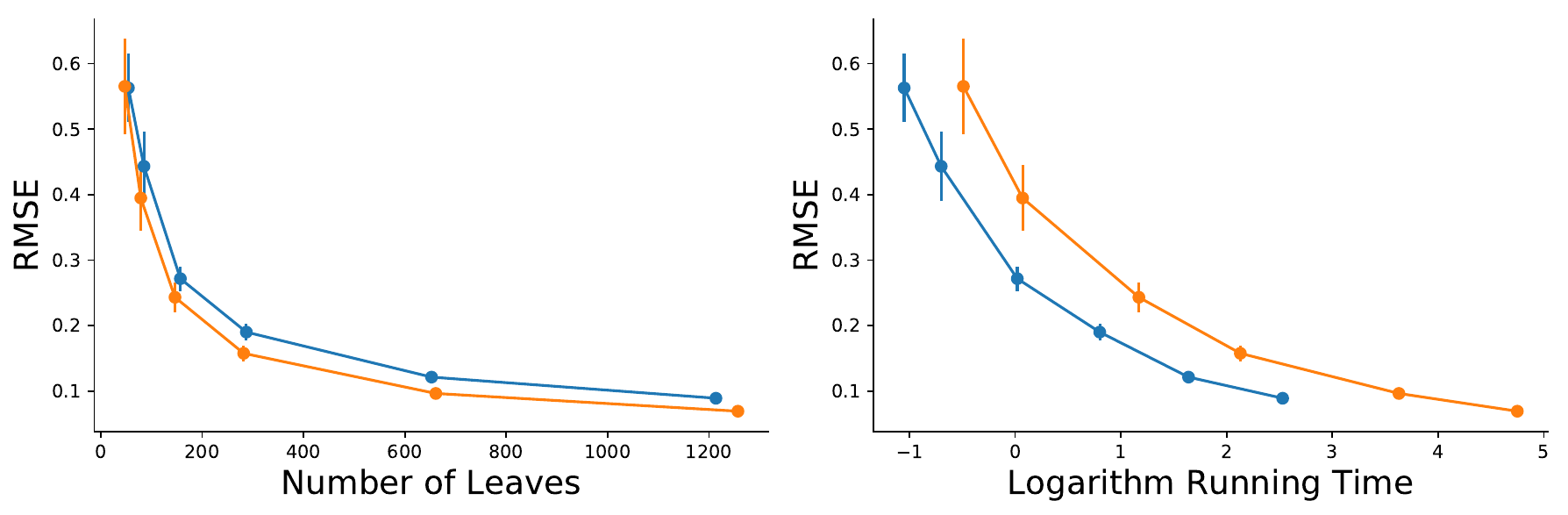}
\caption{Left two sub-figures: RMSE and number of leaf nodes comparisons for the Mondrian Forest and the online BSP-Forests under different settings of the budget sequences. Right two sub-figures: RMSE comparison under different number of datapoints. Points in each curve refers to the case of $N=300, 500, 1,000, 2,000, 5,000, 10,000$ datapoints.}
\label{fig:cut_efficiency}
\end{figure*}

\subsection{Classification}
We examine the classification accuracy of the online BSP-Forest through four real datasets~\cite{lib_svm_dataset} used in~\cite{lakshminarayanan2016mondrian,consistency_online_BF}: satimages~($N=4,435$), letter~($N=15,000$), dna~($N=2,000$) and usps~($N=7,291$). For all these $4$ datasets, we use Principle Component Analysis~(PCA) to extract $4$ principle components before deploying the models. As we have mentioned that the online BSP-Forest is suitable for low dimensional case, this dimensional reduction technique helps us to largely reudce the computational cost~(e.g. DNA dataset with $256$ features). For each dataset, we use different ratios of the whole data as the training data and use the rest as the testing data. For the label prediction, we would first use the majority vote to determine the class label of the nodes in each tree and then again use the majority vote over the datapoints' covered nodes to predict the label for the data point. 

The detail classification performance~(under different ratios of training data) for each dataset is displayed in Figure~\ref{fig:real_data_classifer}. Except for the dna dataset, the performance of the online-BSPF is quite competitive to RF and ERT and much better than the Mondrian Forest algorithm in other three datasets. As RF, ERT and MF proceed each cut using one feature only, the effectiveness of oblique cuts can be verified even for the online setting.

\subsection{Regression}
\paragraph{Friedman's Function (simulated data):} The performance of the online BSP-Forest is first evaluated on the Friedman's function~\cite{friedman1991, chipman2010bart, linero2017bayesian}. In this setting, each datapoint $\pmb{x}'=(x_1,\ldots,x_d)$ is generated from a $d$-dimensional uniform distribution, and its label $y$ takes the form: 
{\small\begin{align}
    y=10\sin(\pi x_1x_2)+20\left(x_3-\frac{1}{2}\right)^2+10x_4+5x_5+\epsilon 
\end{align}}
, where $\epsilon\sim\mathcal{N}(0, \sigma^2), \sigma^2=1$.
%and $x_d$ denotes the $d$-th dimension of $\pmb{x}$.
Friedman's function consists of two nonlinear terms, two linear terms and an interaction term.
% between the dimensions. 

%Since we've already known the groundtruth function,
% (Eq.~\ref{eq:friedman_test}), 
%the RMSE (root mean square error) value can be obtained by directly calculating the distance between the predicted value and the true value. 
We compute the RMSE (root mean squared error) under each method as the true function is known, for different numbers of datapoints $N$, and for two different dimensional setups: $d=5$ where all dimensions are informative, and $d=10$ where only the first 5 dimensions are informative.
%We set the comparison of RMSE under different number of datapoints. Also, we set two cases of the data: one with $5$ informative dimensions only, another with $10$ where only the first $5$ dimensions are informative ones. 
The results are shown in Figure~\ref{fig:Friedman_result}. The online BSP-Forest performs better (in RMSE) than the two (online) Mondrian-Forests, with $\sim0.1$ improvement in RMSE.

Since the three online algorithms are implemented independently of the data labels, it is not surprising that they perform worse than the batch trained RF and ERT. The performance differential is greater when $5$ noisy dimensions ($d=10$) are added. Because the online-BSPF purely uses the generative process to generate the tree structure, it distinguishes between meaningful and noisy dimensions less well. Similar observations have been reported in~\cite{lakshminarayanan2016mondrian}. We may overcome such a performance deficit by increasing the budget, which will permit additional 
%We may increase the value of the budget sequence to address this issue. This will still permit sufficient 
cuts in meaningful dimensions in spite of some less useful cuts in noisy dimensions. 
%Note that does not contradict our earlier claim that the distribution of the online BSP-Forest is the same as that of the batch BSP-Forest. With sufficiently many cuts %(and marginalie the noisy dimensions, 
%their distributions are still the same.

%\textcolor{red}{[This is fine, but then it raises questions about our earlier claim that the distribution of the online BSP is the same as that of the batch BSP. It rather seems to be the case that they are not (and can not) be the same? The resulting distribution is also dependent on the sequence of the data, which is another worry.]}

\paragraph{UK Apartment Price data:} The performance of the online BSP-Forest is examined on the UK Apartment Price Data~\cite{largegaussianprocess}, which records the 
%monthly 
selling price of $122,341$ apartments in the UK from February to October 2012. We take the apartment price as the target label data, and use GPS coordinates~(coarsely estimated based on the apartment's postcode and GPS database) as the feature data ($\pmb{x}$). 

To evaluate the performance of the online BSP-Forest, we randomly sample $20\%$ of the data as the training data and predict the price values on the remaining data. Row $1$ in Table~\ref{table_4} displays the RMSE values of each method. 
%We can see that although RF and ERT perform better than the three online learning algorithms, their RMSE values are similar. 
While, as before, the  RF and ERT methods perform better than the three online learning algorithms, their RMSE values are similar. Also, the online BSP-Forest outperforms the more directly comparable Mondrian Forest variants.

\paragraph{Airline Delay data:}
We finally analyse the regression performance  of the online BSP-Forest on the airline delay data~\cite{largegaussianprocess}. 
We wish to predict flight delay time ($y$) based on $8$ factors ($\pmb{x}$). Following~\cite{distributed_gaussian_processes,lakshminarayanan2016mondrian}, , these factors are set to include the age of the plane, prospective flight distance, prosepctive flight airtime, departure time, arrival time, day of the week, day of the month and month of the year.

Following the settings of ~\cite{distributed_gaussian_processes,lakshminarayanan2016mondrian}, we set the number of training datapoints as $N=100K$, $400K$ and $1$ million, and specify the test data  as the $100K$ observations immediately following the training data. Rows 2--4 in Table~\ref{table_4} display each methods' RMSE values. Similar conclusions can be drawn as for the previous analyses: while the RF and ERT perform better than any the online algorithm,  the online BSP-Forest is the best performing online algorithm. In the exceptional case of $N=100K$, where the MF with increasing budget performs better than the online BSP-Forest, the former variance~($1.2$) is larger than the later~($0.8$).

%Consistently with the previous analyses, although the online BSP-Forest is not the best performer when comparing to RF and ERT (however it can outperform either of these methods depending on the dataset), it consistently performs better than the online Mondrian Forest algorithms in terms of classification accuracy. 

\subsection{Efficiency of the cutting strategy} \label{exp:cut_efficiency}
Cutting efficiency is a primary advantage of the online BSP-Forest. We quantitatively explore this efficiency issue through the simple function: $y= 10\sin(\pi x_1x_2)+0.2\epsilon,\epsilon\sim\mathcal{N}(0, 1)$. Two different cases are investigated: (1) different budget sequences. We set $N=5,000$ and the budget sequence as $\tau_n=0.1n^{1/4},0.2n^{1/4}, \ldots, 1.5n^{1/4}$; (2) different number of datapoints. Based on the results of case~(1), we set $\tau_n=n^{1/4}$ and let the number of datapoints as $N=300, 500, 1,000, 2,000, 5,000, 10,000$ for case~(2). Performance results are shown in Figure~\ref{fig:cut_efficiency}. 

For case~(1), the left two sub-figures confirm that, for different budget sequences, the online BSP-Forest can always obtain better RMSE performance than the Mondrian Forest. The numbers of leaves are also similar between these two algorithms. It is noted that, as each node is restricted to be cut only if it contains more than $3$ datapoints, these algorithms' performance becomes stable even for large value settings of budget sequences. For case~(2), the right two sub-figures indicate the consistent better performance of the online-BSPF for different number of datapoints, with the price of requiring a bit longer running time~(due to the computational cost of $\mathcal{O}(d^2)$).

%%%%%%%%%%%%%%%%%%%%%%%%%%%%%%%%%%%%%%%%%%%%%%
%%%%%%%%%%%%%%%%%%%%%%%%%%%%%%%%%%%%%%%%%%%%%%
\section{Conclusion \& Future work}
%%%%%%%%%%%%%%%%%%%%%%%%%%%%%%%%%%%%%%%%%%%%%%
%%%%%%%%%%%%%%%%%%%%%%%%%%%%%%%%%%%%%%%%%%%%%%
In this paper we have developed an online BSP-Forest framework that addresses the scalability and theoretical challenges of the online BSP-Forest algorithm. Through a non-trivial scheme for efficiently incorporating sequential  datapoints, the model construction for the online BSP-Forest follows the same distribution as that for the batch setting. By using more than one dimension to cut the underlying space, the online BSP-Forest is a demonstrably efficient model for classification and regression tasks. Our experimental results (Section \ref{section:results}) verify that the online BSP-Forest consistently outperforms the Mondrian-Forest, and is competitive with other batch random forest models.

For the future work, the Random Tesselation Forest~\cite{random_tessellation_forests} has extended the BSP-Tree Process by generating arbitrary sloped cutting hyperplanes in $d$-dimensional spaces. Extending the Random Tesselation Forest to the online learning setting and comparing with the online BSP-Forest would be interesting work. Recently, \cite{o2020stochastic} uses the technique of iteration stable (STIT) tessellations to efficiently generate arbitrary sloped cuts. It might be possible to use their method to further reduce the $\mathcal{O}(d^2)$ factor to the computational cost of the online BSP-Forest and make the online BSP-Forest suitable for high-dimensional data as well.

\section*{Acknowledgements}
Xuhui Fan and Scott A.~Sisson are supported by the Australian Research Council through the Australian Centre of Excellence in Mathematical and Statistical Frontiers (ACEMS, CE140100049), and Scott A.~Sisson through the Discovery Project Scheme (DP160102544). Bin Li is supported by Shanghai Municipal Science \& Technology Commission (16JC1420401) and the Program for Professor of Special Appointment (Eastern Scholar) at Shanghai Institutions of Higher Learning.

\bibliography{Xuhui_Machine_Learning}

\begin{thebibliography}{}

\bibitem[Airoldi et~al., 2009]{airoldi2009mixed}
Airoldi, E., Blei, D., Fienberg, S., and Xing, E. (2009).
\newblock Mixed membership stochastic blockmodels.
\newblock In {\em NIPS}, pages 33--40.

\bibitem[Arlot and Genuer, 2014]{arlot2014analysis}
Arlot, S. and Genuer, R. (2014).
\newblock Analysis of purely random forests bias.
\newblock {\em arXiv preprint arXiv:1407.3939}.

\bibitem[Balog and Teh, 2015]{balog2015mondrian}
Balog, M. and Teh, Y.~W. (2015).
\newblock The mondrian process for machine learning.
\newblock {\em arXiv preprint arXiv:1507.05181}.

\bibitem[Biau et~al., 2008]{biau2008consistency}
Biau, G., Devroye, L., and Lugosi, G. (2008).
\newblock Consistency of random forests and other averaging classifiers.
\newblock {\em Journal of Machine Learning Research}, 9(Sep):2015--2033.

\bibitem[Breiman, 2000]{breiman2000some}
Breiman, L. (2000).
\newblock Some infinity theory for predictor ensembles.
\newblock Technical report, Technical Report 579, Statistics Dept. UCB.

\bibitem[Breiman, 2001]{breiman2001random}
Breiman, L. (2001).
\newblock Random forests.
\newblock {\em Machine Learning}, 45(1):5--32.

\bibitem[Chang and Lin, 2011]{lib_svm_dataset}
Chang, C.-C. and Lin, C.-J. (2011).
\newblock Libsvm: A library for support vector machines.
\newblock {\em ACM Trans. Intell. Syst. Technol.}, 2(3):27:1--27:27.

\bibitem[Chipman et~al., 2010]{chipman2010bart}
Chipman, H.~A., George, E.~I., and McCulloch, R.~E. (2010).
\newblock Bart: Bayesian additive regression trees.
\newblock {\em The Annals of Applied Statistics}, 4(1):266--298.

\bibitem[Crane, 2012]{CraneProjectiveSystem}
Crane, H. (2012).
\newblock {\em Infinitely Exchangeable Partition, Tree and Graph-valued
  Stochastic Processes}.
\newblock PhD thesis.

\bibitem[Deisenroth and Ng, 2015]{distributed_gaussian_processes}
Deisenroth, M.~P. and Ng, J.~W. (2015).
\newblock Distributed {Gaussian} processes.
\newblock In {\em ICML}, pages 1481--1490.

\bibitem[Denil et~al., 2013]{consistency_online_BF}
Denil, M., Matheson, D., and Freitas, N. (2013).
\newblock Consistency of online random forests.
\newblock In {\em ICML}, volume~28, pages 1256--1264.

\bibitem[Devroye et~al., 2013]{devroye2013probabilistic}
Devroye, L., Gy{\"o}rfi, L., and Lugosi, G. (2013).
\newblock {\em A probabilistic theory of pattern recognition}, volume~31.
\newblock Springer Science \& Business Media.

\bibitem[Domingos and Hulten, 2000]{domingos2000mining}
Domingos, P. and Hulten, G. (2000).
\newblock Mining high-speed data streams.
\newblock In {\em Kdd}, volume~2, page~4.

\bibitem[Fan et~al., 2019a]{SDREM}
Fan, X., Li, B., Sisson, S., Li, C., and Chen, L. (2019a).
\newblock Scalable deep generative relational model with high-order node
  dependence.
\newblock In {\em NeurIPS}, pages 12637--12647.

\bibitem[Fan et~al., 2018a]{pmlr-v84-fan18b}
Fan, X., Li, B., and Sisson, S.~A. (2018a).
\newblock The binary space partitioning-tree process.
\newblock In {\em AISTATS}, volume~84 of {\em Proceedings of Machine Learning
  Research}, pages 1859--1867.

\bibitem[Fan et~al., 2018b]{NIPS2018_RBP}
Fan, X., Li, B., and Sisson, S.~A. (2018b).
\newblock Rectangular bounding process.
\newblock In {\em NeurIPS}, pages 7631--7641.

\bibitem[Fan et~al., 2019b]{pmlr-v89-fan18a}
Fan, X., Li, B., and Sisson, S.~A. (2019b).
\newblock Binary space partitioning forests.
\newblock In {\em AISTATS}, volume~89, pages 3022--3031.

\bibitem[Fan et~al., 2016a]{xuhui2016OstomachionProcess}
Fan, X., Li, B., Wang, Y., Wang, Y., and Chen, F. (2016a).
\newblock {The Ostomachion Process}.
\newblock In {\em AAAI Conference on Artificial Intelligence}, pages
  1547--1553.

\bibitem[Fan et~al., 2016b]{fan2016copula}
Fan, X., Xu, R. Y.~D., and Cao, L. (2016b).
\newblock Copula mixed-membership stochastic block model.
\newblock In {\em IJCAI}.

\bibitem[Friedman, 1991]{friedman1991}
Friedman, J.~H. (1991).
\newblock Multivariate adaptive regression splines.
\newblock {\em The Annals of Statistics}, 19:1--67.

\bibitem[Ge et~al., 2019]{random_tessellation_forests}
Ge, S., Wang, S., Teh, Y.~W., Wang, L., and Elliott, L. (2019).
\newblock Random tessellation forests.
\newblock In {\em NeurIPS}, pages 9571--9581.

\bibitem[Genuer, 2012]{genuer2012variance}
Genuer, R. (2012).
\newblock Variance reduction in purely random forests.
\newblock {\em Journal of Nonparametric Statistics}, 24(3):543--562.

\bibitem[Geurts et~al., 2006]{geurts2006extremely}
Geurts, P., Ernst, D., and Wehenkel, L. (2006).
\newblock Extremely randomized trees.
\newblock {\em Machine learning}, 63(1):3--42.

\bibitem[Graham, 1972]{Graham1972}
Graham, R.~L. (1972).
\newblock An efficient algorithm for determining the convex hull of a finite
  planar set.
\newblock {\em Information Processing Letters}, 1(4):132--133.

\bibitem[Hensman et~al., 2013]{largegaussianprocess}
Hensman, J., Fusi, N., and Lawrence, N.~D. (2013).
\newblock Gaussian processes for big data.
\newblock In {\em UAI}, pages 282--290.

\bibitem[Karrer and Newman, 2011]{karrer2011stochastic}
Karrer, B. and Newman, M.~E. (2011).
\newblock Stochastic blockmodels and community structure in networks.
\newblock {\em Physical Review E}, 83(1):016107.

\bibitem[Kemp et~al., 2006]{kemp2006learning}
Kemp, C., Tenenbaum, J.~B., Griffiths, T.~L., Yamada, T., and Ueda, N. (2006).
\newblock Learning systems of concepts with an infinite relational model.
\newblock In {\em AAAI}, volume~3, pages 381--388.

\bibitem[Lakshminarayanan et~al., 2014]{LakRoyTeh2014a}
Lakshminarayanan, B., Roy, D.~M., and Teh, Y.~W. (2014).
\newblock {Mondrian} forests: Efficient online random forests.
\newblock In {\em NIPS}, pages 3140--3148.

\bibitem[Lakshminarayanan et~al., 2016]{lakshminarayanan2016mondrian}
Lakshminarayanan, B., Roy, D.~M., and Teh, Y.~W. (2016).
\newblock Mondrian forests for large-scale regression when uncertainty matters.
\newblock In {\em AISTATS}, pages 1478--1487.

\bibitem[Li et~al., 2009]{Li_transfer_2009}
Li, B., Yang, Q., and Xue, X. (2009).
\newblock Transfer learning for collaborative filtering via a rating-matrix
  generative model.
\newblock In {\em ICML}, pages 617--624.

\bibitem[Linero and Yang, 2017]{linero2017bayesian}
Linero, A.~R. and Yang, Y. (2017).
\newblock Bayesian regression tree ensembles that adapt to smoothness and
  sparsity.
\newblock {\em arXiv preprint arXiv:1707.09461}.

\bibitem[Mourtada et~al., 2017]{consistencyMondrianforest}
Mourtada, J., Ga\"{\i}ffas, S., and Scornet, E. (2017).
\newblock Universal consistency and minimax rates for online mondrian forests.
\newblock In {\em NIPS}, pages 3758--3767.

\bibitem[Mourtada et~al., 2018]{mourtada2018minimax}
Mourtada, J., Ga{\"\i}ffas, S., and Scornet, E. (2018).
\newblock Minimax optimal rates for mondrian trees and forests.
\newblock {\em arXiv preprint arXiv:1803.05784}.

\bibitem[Nakano et~al., 2014]{nakano2014rectangular}
Nakano, M., Ishiguro, K., Kimura, A., Yamada, T., and Ueda, N. (2014).
\newblock Rectangular tiling process.
\newblock In {\em ICML}, pages 361--369.

\bibitem[Nowicki and Snijders, 2001]{nowicki2001estimation}
Nowicki, K. and Snijders, T.~A. (2001).
\newblock Estimation and prediction for stochastic block structures.
\newblock {\em Journal of the American Statistical Association},
  96(455):1077--1087.

\bibitem[O'Reilly and Tran, 2020]{o2020stochastic}
O'Reilly, E. and Tran, N. (2020).
\newblock Stochastic geometry to generalize the mondrian process.
\newblock {\em arXiv preprint arXiv:2002.00797}.

\bibitem[Pedregosa et~al., 2011]{scikit-learn}
Pedregosa, F., Varoquaux, G., Gramfort, A., Michel, V., Thirion, B., Grisel,
  O., Blondel, M., Prettenhofer, P., Weiss, R., Dubourg, V., et~al. (2011).
\newblock Scikit-learn: Machine learning in {P}ython.
\newblock {\em Journal of Machine Learning Research}, 12(Oct):2825--2830.

\bibitem[Porteous et~al., 2008]{porteous2008multi}
Porteous, I., Bart, E., and Welling, M. (2008).
\newblock {Multi-HDP}: A non parametric {Bayesian} model for tensor
  factorization.
\newblock In {\em AAAI}, pages 1487--1490.

\bibitem[Roy, 2011]{roy2011thesis}
Roy, D.~M. (2011).
\newblock {\em Computability, Inference and Modeling in Probabilistic
  Programming}.
\newblock PhD thesis, MIT.

\bibitem[Roy et~al., 2007]{roy2007learning}
Roy, D.~M., Kemp, C., Mansinghka, V., and Tenenbaum, J.~B. (2007).
\newblock Learning annotated hierarchies from relational data.
\newblock In {\em NIPS}, pages 1185--1192.

\bibitem[Roy and Teh, 2009]{roy2009mondrian}
Roy, D.~M. and Teh, Y.~W. (2009).
\newblock The {Mondrian} process.
\newblock In {\em NIPS}, pages 1377--1384.

\bibitem[Saffari et~al., 2009]{On_Line_RF}
Saffari, A., Leistner, C., Santner, J., Godec, M., and Bischof, H. (2009).
\newblock On-line random forests.
\newblock In {\em ICCV Workshops}, pages 1393--1400.

\bibitem[Yule, 1925]{fye1925mathematical}
Yule, G.~U. (1925).
\newblock A mathematical theory of evolution based on the conclusions of {D}r.
  {J}c {W}illis, {FRS}.
\newblock {\em Philosophical Transactions of the Royal Society B}.

\end{thebibliography}
\bibliographystyle{apalike}

\appendix

\section{Algorithm of Generating A Cut in $\Diamond'$ and Not Crossing into $\Diamond$}
\begin{algorithm}[H]
\caption{GenCutNorCross($\Diamond', \Diamond$)} 
\label{algo:CutConvexHull}
\begin{algorithmic}[1]
\STATE $(d_1^*, d_2^*)\sim\textrm{Cat}(\tilde{L}_{(1, 2)}(\Diamond')-\tilde{L}_{(1, 2)}(\Diamond),\ldots$, $ \tilde{L}_{(d-1, d)}(\Diamond')-\tilde{L}_{(d-1, d)}(\Diamond))$
\STATE $\theta\sim p(\theta)\propto |\pmb{l}_{\Pi_{(d_1^*, d_2^*)}(\Diamond')}(\theta)|-|\pmb{l}_{\Pi_{(d_1^*, d_2^*)}(\Diamond)}(\theta)|, \theta\in(0, \pi]$
\STATE Sample $\pmb{u}$ uniformly on $|\pmb{l}_{\Pi_{(d_1^*, d_2^*)}(\Diamond')}(\theta)|-|\pmb{l}_{\Pi_{(d_1^*, d_2^*)}(\Diamond)}(\theta)|$
\STATE Form cutting hyperplane based on $(d_1^*, d_2^*), \theta, \pmb{u}$ 
\end{algorithmic}
\end{algorithm}

\section{Some visualisations}
\begin{figure*}[t]
\centering
\includegraphics[width =  0.45 \textwidth]{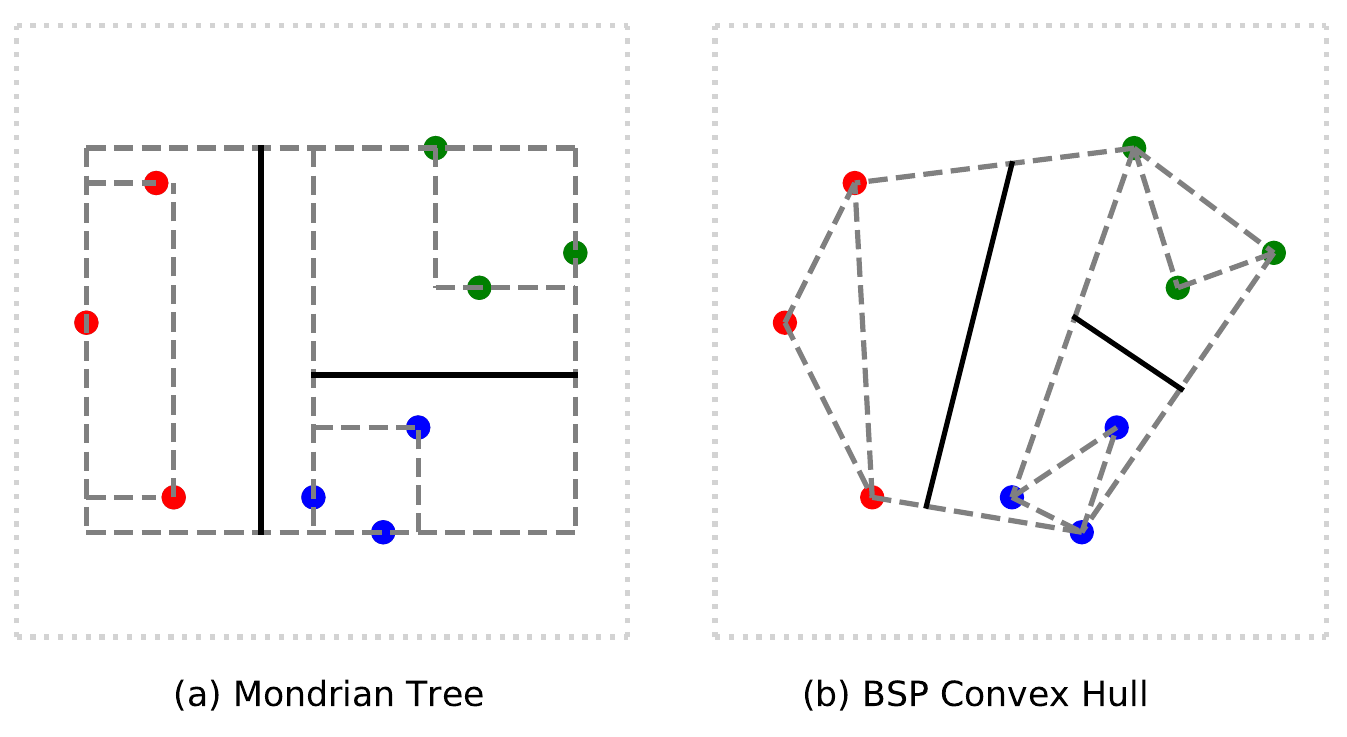}
\quad
\includegraphics[width =  0.38 \textwidth]{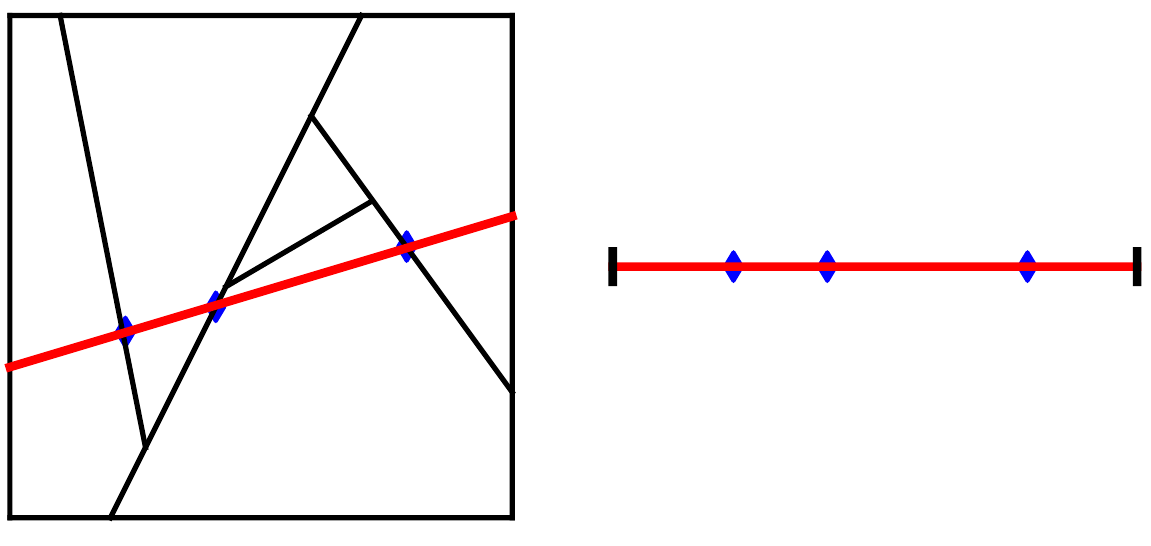}
\caption{Left: example visualization comparing the Mondrian-Tree and BSP-Tree convex hulls. Point colors identify different data labels, and the dotted, dashed and solid lines denote the whole space, the convex hulls and the cuts, respectively. Right: $2d$ visualization of oblique line slice of the BSP-Tree Process partition. Red solid line denotes the oblique line and blue dots represents the intersection points.}
\label{fig:cuts_vis}
\end{figure*}
Left panel of Figure~\ref{fig:cuts_vis} visualizes the difference between the convex hull representation of tree node in the BSP-Tree partition and Mondrian tree partition. Convex hulls are formed recursively, and larger hulls contain smaller ones. The BSP-Tree generates smaller convex hulls than the Mondrian-Tree, which means the BSP-Tree is a ``tight'' representation of the space. Right panel of Figure~\ref{fig:cuts_vis} visualizes oblique line slice of the BSP-Tree Process.

\section{Proof of Lemma 1}
\begin{lemma} \label{lemma:dim_1_poisson_process}
(Oblique line slice) For any oblique line that crosses into the domain of a BSP-Tree process with budget $\tau$, its intersection points with the partition forms a homogeneous Poisson process with intensity $2\tau$.\end{lemma}
\begin{proof}
The self-consistent property of the BSP-Tree process guarantees that this $1$-dimensional slice follows the same way of directly generating a BSP-Tree partition on the line. 

To define the BSP-Tree partition on the line segment, we first consider the BSP-Tree partition in an obtuse triangle. Two vertices of the triangle form a line segment with the length $L$. Another vertex lies between these two vertices and has an $\epsilon$ distance to the line segment. Based on the generative process of the BSP-Tree process, the cost of cut in this triangle follows an Exponential distribution with rate parameter being the perimeter of the triangle, which is $PE=(L_1+L_2)+\sqrt{L_1^2+\epsilon^2}+\sqrt{L_2^2+\epsilon^2}$. As $PE\to 2(L_1+L_2)$ when $\epsilon\to 0$, the cost has an exponential distribution with rate parameter $2L$ accordingly. As a result, the number of cuts follows a Poisson distribution with parameter $\tau\cdot 2L $. The cut position is Uniformly distributed in the line segment. (For each projection in the direction of $\theta$, the crossing point between the cuts and line segment is Uniformly distributed.) This can verify the independent increments of the partition points in the line. 

As the two condition of Poisson process is satisfied, according to Theorem 1.10 in~\cite{balog2015mondrian}, we can get the conclusion.
\end{proof}

\begin{figure*}[t]
\centering
\includegraphics[width =  0.45 \textwidth]{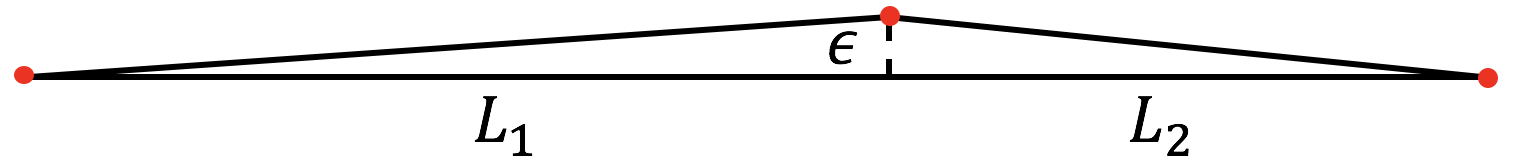}
\caption{Visualization for the $1$-dimansional space case.}
\label{fig: bsp_dim_1_vis}
\end{figure*}

\section{Proof of Theorem 2}
It is noted the main idea of the following proof largely follows the work of~\cite{consistencyMondrianforest}. We make modifications to make the proof suitable to online BSP-Forest case. 
\begin{lemma} \label{lemma:cell_diameter}
(Block diameter) Let $\pmb{x}\in[0, 1]^d$, and let $D(\pmb{x})$ be the ${L}^2$-diameter of the block containing $\pmb{x}$ in the BSP-Tree partition with budget $\tau/2$. If $\tau\to\infty$, then $D_{\tau}(\pmb{x})\to 0$ in probability. More precisely, for every $\delta, \tau>0$, we have
$$
\mathbb{P}(D_{\tau}(\pmb{x})\ge\delta)\le d(1+\frac{\tau\delta}{\sqrt{d}})\exp(-\frac{\tau\delta}{\sqrt{d}})
,\quad\mathbb{E}[D_{\tau}(\pmb{x})^2]\le\frac{4d}{\tau^2}.
$$
\end{lemma}

\begin{proof}
Let $\square_{\tau}(\pmb{x})$ denotes the block of a Binary Space Partitioning-Tree partition containing $\pmb{x}\in[0, 1]^d$. In the space of $[0, 1]^d$, we can build up $d$ orthogonal basises to describe the block $\square_{\tau}(\pmb{x})$, with one basis is in the direction of largest diameter in $\square_{\tau}(\pmb{x})$. While it is obvious that rotations of the bock will not affect the diameter, w.l.o.g., we rotate the block and treat the direction with largest diameter as dimension $1$. By definition, the $L^{\infty}$-norm diameter $D_{\tau}(\pmb{x})$ of $\square_{\tau}(\pmb{x})$ is $\max\{\square_{\tau}^{(d')}(\pmb{x})\}_{d'}$. While recording these smallest and largest interceptions in these rotated dimensions as $\{L_{\tau}^{(d')},R_{\tau}^{(d')}\}_{d'}$, all of the random variables $R_{\tau}(\pmb{x})-L_{\tau}(\pmb{x})$ have the same distribution, it suffices to consider $D_{\tau}^{(1)}(\pmb{x})=R^{(1)}_{\tau}(\pmb{x})-L^{(1)}_{\tau}(\pmb{x})$. 

In this rotated block, consider the segment $I^{(1)}(\pmb{x})=[0,\sqrt{d}]\times \pmb{x}^{-i}$ containing $\pmb{x}$, and denote $\phi_{\tau}^1(\pmb{x})\subset[0, \sqrt{d}]$ the restriction of the partition to $I^{(1)}(x)$. Note that $R_{\tau}^{(1)}(\pmb{x})$~($L_{\tau}^{(1)}(\pmb{x})$) is the lowest~(highest) element of $\phi_{\tau}^{(1)}(\pmb{x})$ that is larger~(smaller) than $x_1$, and is equal to $\sqrt{d}$~($0$) if $\phi_{\tau}^{(1)}(\pmb{x})\cap[x_1, \sqrt{d}]$~($\phi_{\tau}^{(1)}(\pmb{x})\cap[0, x_1]$) is empty. By Theorem 1 and Lemma 1, $\phi_{\tau}(\pmb{x})$ is a Poisson process with intensity $\tau$. 

This implies the distribution of $(L_{\tau}^{(1)}(\pmb{x}), R_{\tau}^{(1)}(\pmb{x}))$ is the same as that of $(\tilde{L}_{\tau}^{(1)}(\pmb{x})\lor 0, \tilde{R}_{\tau}^{(1)}(\pmb{x})\land\sqrt{d})$, where $\tilde{\phi}_{\tau}^1(\pmb{x})$ is a Poisson process on $\mathbb{R}$ with intensity $\tau$, and $\tilde{L}_{\tau}^{(1)}(\pmb{x})=\sup(\tilde{\phi}_{\tau}^{(1)}(x_1)\cup(-\infty, x_1)), \tilde{R}_{\tau}^{(1)}(\pmb{x})=\inf(\tilde{\phi}_{\tau}^{(1)}(x_1)\cap(x_1, \infty))$. By the property of the Poisson point process, this implies that $x_1-L_{\tau}^{(1)}(\pmb{x}), R_{\tau}^{(1)}(\pmb{x})-x_1\overset{d}{=}(E_1, E_2)$, where $E_1, E_2$ are independent exponential random variables with parameter $\tau$. $D_{\tau}^{(1)}(\pmb{x})=R_{\tau}^{(1)}(\pmb{x})-x_1+x_1-L_{\tau}^{(1)}(\pmb{x})$ is upper bounded by $E_1+E_2\sim\textrm{Gamma}(2, \tau)$. Thus, we have $\forall \delta >0, \mathbb{P}(D_{\tau}^1(\pmb{x}))\ge\delta)\le (1+\tau\delta)e^{-\tau\delta}$ and $\mathbb{E}[D_{\tau}^1(\pmb{x})^2]\le \mathbb{E}(E_1^2)+\mathbb{E}(E_2^2)=\frac{4}{\tau^2}$. The bound of $D_{\tau}(\pmb{x})$ can be obtained by $D_{\tau}(\pmb{x})=\sqrt{\sum_{d'}D_{\tau}^{d'}(\pmb{x})}$. 
\end{proof}
% \begin{figure}[t]
% \centering
% \caption{
% % the whole space is outlined by gray dotted lines, the black solid lines denotes the cuts and the convex hulls (Mondrian Trees) are described by dashed lines. 
% Convex hulls are formed recursively, and larger hulls contain smaller ones.
% %They are formed in a recursive nature and a larger convex hull (Mondrian Tree) would contain smaller ones. 
% The BSP-Tree generates smaller convex hulls than the Mondrian-Tree, which means the BSP-Tree is a ``tight'' representation of the space.}
% \label{fig:mt_vs_convex_hull}
% \end{figure}

\begin{lemma} \label{lemma:expected_k}
If $K_{\tau}$ denotes the number of cuts in the BSP-Tree process, we have $\mathbb{E}[K_{\tau}]\le(1+\tau)^de^{d(d-1)}$.  
%\textcolor{red}{checking: $e\approx 2.71$, or typo?}
\end{lemma}
\begin{proof}
Let $\square\subset[0, 1]^d$ be an arbitrary block, and let $K_{\tau}^{\square}$ denotes the number of splits performed in the BSP-Tree process with budget value $\tau/2$. As shown in~\cite{pmlr-v84-fan18b}\cite{pmlr-v89-fan18a}, the waiting time of a cut occurs in a leaf node $\phi$ of the BSP-Tree process follows an exponential distribution of rate $L(\square_{\phi})\le L(\square)$, where $L(\square)$ denotes the perimeter of the block $\square$. The number of leaves $K_t+1\ge K_t$ at time $t$ is dominated by the number of individuals in a Yule process with rate $L(\square)$~\cite{fye1925mathematical}. Thus, we have: $\mathbb{E}(K_{\tau}^\square)\le e^{\tau L(\square)}$. 

Considering the covering $\mathcal{C}$ of $\square$ by a  regular grid of $\lceil\tau\rceil^d$ boxes obtained by equally dividing each coordinate of $\square$ in $\lceil\tau\rceil$ parts. Each cut in $\square$ will induce a split in at least one box $C$in $\mathcal{C}$ and $B_{\tau}^{\mathcal{C}}$ is also a BSP-Tree process in box $C$ (due to the self-consistency of the BSP-Tree process), we have: $\mathbb{E}(K_{\tau}^\square)\le\sum_{C\in\mathcal{C}}\mathbb{E}(K_{\tau}^C)\le \lceil\tau\rceil^de^{\tau \frac{L(\square)}{\lceil\tau\rceil}}\le (\tau+1)^de^{L(\square)}$. 
\end{proof}

\begin{lemma} \label{lemma_4}
Assume that the total number of splits $K_{\tau}$ performed by the BSP-Tree partition satisfies $\lim_{n\to\infty}\mathbb{E}(K_{\tau_n})/n\to 0$. For $N_{n}(\pmb{x})$ being the number of datapoints in $\pmb{x}_{1:N}$ fall in $A_{\tau_n}(\pmb{x})$, we have $N_{n}(\pmb{x})\to\infty$ in probability.
\end{lemma}
\begin{proof}
We fix $n\ge 1$, and conditionally on the BSP-Tree partition  at the budget of $\tau_n$, $B_{\tau_n}$ is independent of $\pmb{x}$ by construction. Note that the number of leaves is $|\mathcal{L}(B_{\tau_n})|=K_{\tau_n}+1$, and $(\square_{\phi})_{\phi\in\mathcal{L}(B_{\tau_n})}$ is the corresponding blocks, where $\phi$ refers tot he leaf node. For $\phi$, we define $N_{\phi}$ to be the number of points among $\pmb{x}_1, \ldots, \pmb{x}_n$ that fall in the cell $\square_{\phi}$. Since $\pmb{x}_1, \ldots, \pmb{x}_n$ are i.i.d., so that the joint distribution of $\pmb{x}_1, \ldots, \pmb{x}_n$ is invariance under the permutation of the $n+1$ datapoints, conditionally on the set $S={\pmb{x}_1, \ldots, \pmb{x}_n}$ the probability that $\pmb{x}$ falls in the block $\square_{\phi}$. Therefore, for each $t>0$, we have:
\begin{align}
\mathbb{P}(N_{n}(\pmb{x})\le t)=&\mathbb{E}\{\mathbb{P}(N_n(\pmb{x})\le t|S, B_{\tau_n})\}\nonumber\\
=&\mathbb{E}\left\{\sum_{\phi\in\mathcal{L}(B_{\tau_n}):N_{\phi}\le t}\frac{N_{\phi}}{n+1}\right\} \\ \nonumber
\le & \mathbb{E}\left\{\frac{t|\mathcal{L}(B_{\tau_n})|}{n+1}\right\} 
=  \frac{t(\mathbb{E}(K_{\tau_n})+1)}{n+1}
\end{align}
which tends to $0$ as $n\to\infty$. 
\end{proof}

Before proving Theorem 2, we first invoke a consistency theorem~(Theorem 6.1 in \cite{devroye2013probabilistic} and we use $\square$ to denote the block for notation consistency)
\begin{customthm}{4}\label{invoke_theorem}
Consider a sequence of randomised tree classifiers $(\tilde{g}_n(\cdot,Z))$, grown independently of the labels $Y_1, \ldots, Y_n$. For $\pmb{x}\in[0, 1]^d$, denote $\square_n(\pmb{x})=\square_n(\pmb{x}, Z)$ the block containing $\pmb{x}$, $D_n(\pmb{x})$ its diameter and $N_n(\pmb{x})=N_{n}(\pmb{x}, Z)$ the number of input vectors among $\pmb{x}_1, \ldots, \pmb{x}_n$ that fall in $\square_n(\pmb{x})$. Assume that, if $\pmb{x}$ is drawn from the distribution with the following conditions:
\begin{enumerate}
    \item $\lim_{n\to\infty}D_n(\pmb{x})\to 0$ in probability;
    \item $\lim_{n\to\infty}N_n(\pmb{x})\to\infty $ in probability.
\end{enumerate}
Then the tree classifier $\tilde{g}_n$ is consistent.
\end{customthm}

The proof of Theorem 2 is:
\begin{proof}
We can show the two conditions in Theorem~\ref{invoke_theorem} are satisfied. First, Lemma 1 ensures that, if $\tau_n\to\infty$, $D_{\tau_n}(\pmb{x})\to 0$ in probability for every $\pmb{x}\in[0, 1]^d$. In particular, for every $\delta>0$, we have $\mathbb{P}(\square_{\tau_n}(\pmb{x})\ge\delta)=\int_{[0, 1]^d}\mathbb{P}(D_{\tau_n}(\pmb{x})>\delta)\mu(d\pmb{x})\to 0$ as $n\to\infty$ by the dominated convergence theorem. 

Since Lemma~\ref{lemma_4} provides the proof for the second condition, the proof of Theorem 2 is concluded.
\end{proof}

\section{Proof of Theorem 3}
It is noted the main idea of the following proof largely follows the work of~\cite{consistencyMondrianforest}. We make modifications to make the proof suitable to online BSP-Forest case. 
\begin{proof}
By the convexity of the quadratic loss function and the fact that all the BSP-Tree has the same distribution, we have that $\mathbb{E}\left[(g(\pmb{x})-\widehat{g}_n(\pmb{x}))^2\right]\le \frac{1}{m}\sum_{k=1}^m\mathbb{E}\left[(g(\pmb{x})-\widehat{g}_{n,k}(\pmb{x}))^2\right]=\mathbb{E}\left[(g(\pmb{x})-\widehat{g}_{n,1}(\pmb{x}))^2\right]$. Thus, we can prove the result for a single treee algorithm to get the conclusion.
\end{proof}

We firs write use the bias-variance decomposition of the quadratic loss by:
\begin{align}
R(\widehat{f}_n)=\mathbb{E}\left[(f(\pmb{x})-\bar{f}_n(\pmb{x}))^2\right]+\mathbb{E}\left[(\widehat{f}_n(\pmb{x})-\bar{f}_n(\pmb{x}))^2\right]
\end{align}
where $\bar{f}_n(\pmb{x}):=\mathbb{E}\left[f(\pmb{x}|\pmb{x}\in A_n(\pmb{x}))\right]$ denotes the groundtruth label value for the block containing $\pmb{x}$. The first term is bias and it measures the closeness of $f(\pmb{x})$ to the best approximator $\bar{f}_n(\pmb{x})$ (of which the label value is constant on the block containing $\pmb{x}$). The second term is variance and it measures the closeness of he best approximator $\bar{f}_n(\pmb{x})$ to the empirical approximator $\widehat{f}_n(\pmb{x})$. 

For the bias term, we have:
\begin{align}
|f(\pmb{x})-\bar{f}_n(\pmb{x})|\le&|\frac{1}{\mu(A_n(\pmb{x}))}\int_{A_{n}(\pmb{x})}(f(\pmb{x})-\bar{f}_n(\pmb{z}))\mu(d\pmb{z})|\nonumber \\
\le & \sup_{\pmb{z}\in A_n(\pmb{x})}|f(\pmb{x})-f(\pmb{z})|\nonumber \\
\le & L\sup_{\pmb{z}\in A_n(\pmb{x})}\|\pmb{x}-\pmb{z}\|_2=L\cdot D_{n}(\pmb{x})
\end{align}
where $D_{n}(\pmb{x})$ is the $l^2$-diameter of $A_{n}(\pmb{x})$. According to the result of Lemma 1, we get:
\begin{align} \label{eq:bias_term}
\mathbb{E}\left[(f(\pmb{x})-\bar{f}_n(\pmb{x}))^2\right]\le L^2\mathbb{E}\left[D_{n}(\pmb{x})^2\right]\le\frac{4dL^2}{\tau_n^2}
\end{align}

For the variance term, based on the Proposition 2 of~\cite{arlot2014analysis}: if $U$ is a random tree partition of the unit space with $k+1$ blocks, we have:
\begin{align}
\mathbb{E}\left[(\bar{f}_U(\pmb{x})-\widehat{f}_U(\pmb{x}))^2\right]\le\frac{k+1}{n}(2\sigma^2+9\|f\|_{\infty})
\end{align}. 
Thus, we can have:
\begin{align}
&\mathbb{E}\left[(\bar{f}_n(\pmb{x})-\widehat{f}_n(\pmb{x}))^2\right]=\sum_{k=0}^{\infty}\mathbb{P}(k)\mathbb{E}\left[(\bar{f}_U(\pmb{x})-\widehat{f}_U(\pmb{x}))^2|k\right]\nonumber \\
\le & \sum_{k=0}^{\infty}\mathbb{P}(k)\frac{k+1}{n}(2\sigma^2+9\|f\|_{\infty})\nonumber \\
= & \frac{\mathbb{E}(K_n)+1}{n}(2\sigma^2+9\|f\|_{\infty})
\end{align}
Based on the result of Lemma 2, we get
\begin{align} \label{eq:variance_term}
&\mathbb{E}\left[(\bar{f}_n(\pmb{x})-\widehat{f}_n(\pmb{x}))^2\right]\nonumber \\
\le & \frac{(1+\tau_n)^de^{d(d-1)}+1}{n}(2\sigma^2+9\|f\|_{\infty})
\end{align}

Combining the result of Eq.~(\ref{eq:bias_term})(\ref{eq:variance_term}), we get:
\begin{align}
& R(\widehat{f}_n)\nonumber \\
\le&\frac{4dL^2}{\tau_n^2} + \frac{(1+\tau_n)^de^{d(d-1)}+1}{n}(2\sigma^2+9\|f\|_{\infty})
\end{align}
Taking $\tau_n=n^{\frac{1}{d+2}}$ can make $R(\widehat{f}_n)$ scales to $\mathcal{O}(n^{-\frac{2}{d+2}})$. 

\section{Additional Experimental Results}
\begin{figure}[t]
\centering
\includegraphics[width =  0.45 \textwidth]{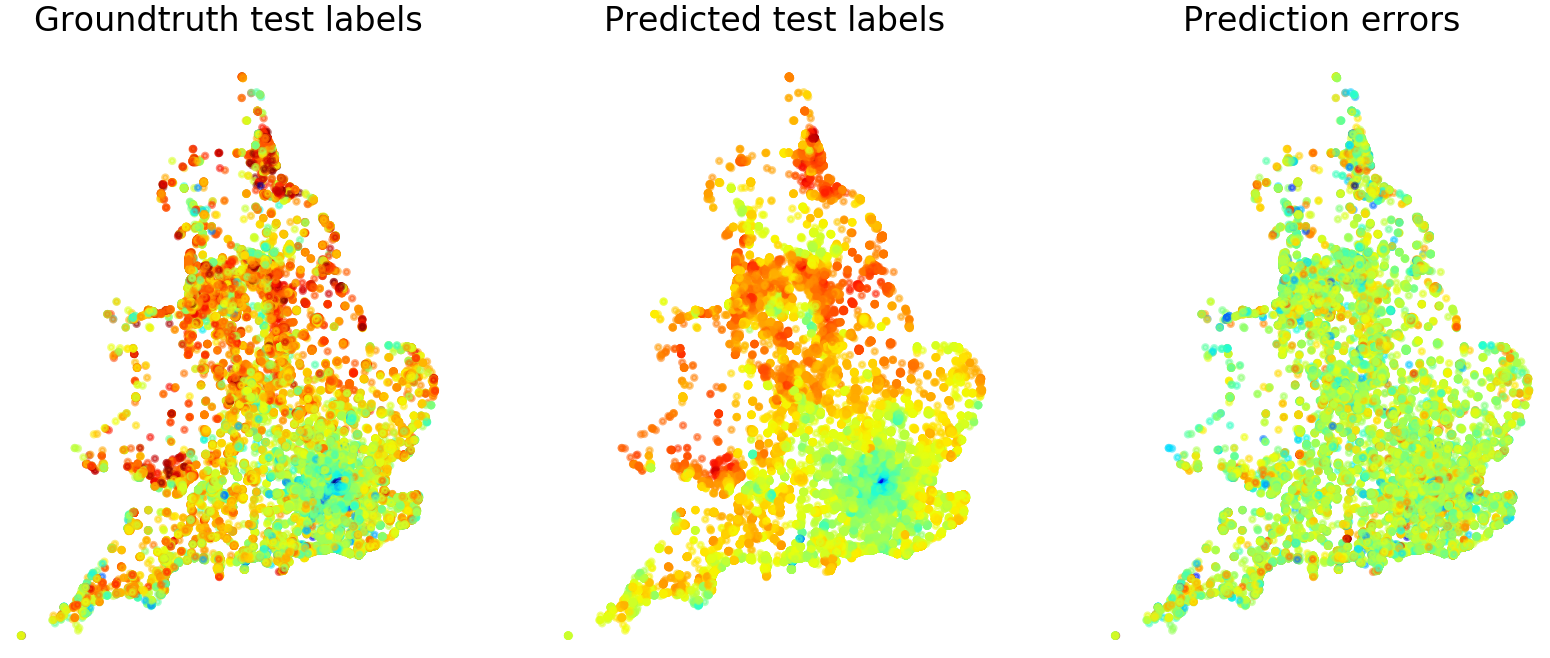}
\caption{{Visualisation of the online BSP-Forest's spatial predictions on the UK Apartment Price data. 
Plots show [L-R] actual test data, predictions, {and prediction errors}.
Red--blue colour denotes low--high prices.}}
\label{fig:Friedman_result}
\end{figure}
Figure~\ref{fig:Friedman_result} illustrates the observed and predicted labels (and their difference) for the UK Apartment Price Data. The online BSP-Forest appears to be able to capture the price variation reasonably well, and provide an accurate prediction of the true test data. Spatially, the prediction error looks broadly pattern free (in colour distribution), indicating that the regression model is adequate for these data. 

\end{document}